\documentclass[fleqn,11pt]{article}
\usepackage{fullpage}
\usepackage{a4wide}

\usepackage[english]{babel}
\usepackage[boxed]{algorithm}
\usepackage[noend]{algorithmic}
\usepackage{amsmath,amssymb,amsfonts,amsthm}
\usepackage[fleqn,tbtags]{mathtools}
\usepackage{lastpage}
\usepackage{bbm}
\usepackage{fancybox}
\usepackage{boxedminipage}
\usepackage{fancyhdr}
\usepackage{graphicx}
\usepackage{subfigure}
\usepackage{epsfig}
\usepackage{color}
\usepackage{xspace}
\usepackage{url}
\usepackage{ifpdf}
\ifpdf
\usepackage{hyperref}
\usepackage[utf8]{inputenc} 
\usepackage{booktabs}       
\else
\usepackage[hypertex]{hyperref}
\fi
\usepackage [autostyle, english = american]{csquotes}

\MakeOuterQuote{"}

\newtheorem{theorem}{Theorem}
\newtheorem{lemma}[theorem]{Lemma}
\newtheorem{corollary}[theorem]{Corollary}
\newtheorem{definition}{Definition}
\newtheorem{proposition}[theorem]{Proposition}
\newtheorem{claim}[theorem]{Claim}

\newtheorem{notation}{Notation}

\makeatletter
\newtheorem*{rep@theorem}{\rep@title}
\newcommand{\newreptheorem}[2]{%
\newenvironment{rep#1}[1]{%
 \def\rep@title{#2 \ref{##1}}%
 \begin{rep@theorem}}%
 {\end{rep@theorem}}}
\makeatother

\newreptheorem{theorem}{Theorem}

\newcommand {\ignore} [1] {}

\def \eqdef {:=}

\def \xk {x^{(k)}}

\DeclareMathOperator{\supp}{supp}

\providecommand{\eqdef}{:=}

\newcommand{\etal}{{\em et al.\ }\xspace}

\newcommand*\samethanks[1][\value{footnote}]{\footnotemark[#1]}

\title{Fully Understanding the Hashing Trick}

\author{Casper Benjamin Freksen\thanks{Computer Science Department. Aarhus University. Supported by a Villum Young Investigator Grant. \texttt{\{cfreksen, lior.kamma\}@cs.au.dk}.} \qquad 
Lior Kamma\samethanks
\qquad 
Kasper Green Larsen \thanks{Computer Science Department. Aarhus University. Supported by a Villum Young
    Investigator Grant and an AUFF Starting Grant. \texttt{larsen@cs.au.dk}. }
}
\begin{document}
\maketitle

\begin{abstract}
Feature hashing, also known as {\em the hashing trick}, introduced by Weinberger \etal (2009), is one of the key techniques used in scaling-up machine learning algorithms. Loosely speaking, feature hashing uses a random sparse projection matrix $A : \mathbb{R}^n \to \mathbb{R}^m$ (where $m \ll n$) in order to reduce the dimension of the data from $n$ to $m$ while approximately preserving the Euclidean norm. Every column of $A$ contains exactly one non-zero entry, equals to either $-1$ or $1$.

Weinberger \etal showed tail bounds on $\|Ax\|_2^2$. Specifically they showed that for every $\varepsilon, \delta$, if $\|x\|_{\infty} / \|x\|_2$ is sufficiently small, and $m$ is sufficiently large, then 
\begin{equation*}\Pr[ \; | \;\|Ax\|_2^2 - \|x\|_2^2\; | < \varepsilon \|x\|_2^2 \;] \ge 1 - \delta \;.\end{equation*}
These bounds were later extended by Dasgupta \etal (2010) and most recently refined by Dahlgaard \etal (2017), however, the true nature of the performance of this key technique, and specifically the correct tradeoff between the pivotal parameters $\|x\|_{\infty} / \|x\|_2, m, \varepsilon, \delta$ remained an open question.

We settle this question by giving tight asymptotic bounds on the exact tradeoff between the central parameters, thus providing a complete understanding of the performance of feature hashing. We complement the asymptotic bound with empirical data, which shows that the constants "hiding" in the asymptotic notation are, in fact, very close to $1$, thus further illustrating the tightness of the presented bounds in practice.
\end{abstract}


\section{Introduction}

{\em Dimensionality reduction} that approximately preserves Euclidean distances is a key tool used as a preprocessing step in many geometric, algebraic and classification algorithms, whose performance heavily depends on the 
dimension of the input. 
Loosely speaking, a distance-preserving dimensionality reduction is an (often random) embedding of a high-dimensional Euclidean space into a space of low dimension, such that the distance between every two points is approximately preserved (with high probability). 
Its applications
range upon nearest neighbor search \cite{AC09, HIM12},
classification and regression \cite{RR08,MM09,SBMD14},
manifold learning \cite{HWB08}
sparse recovery \cite{CR06} 
and numerical linear algebra \cite{CW09,MM13, S06}.
For more applications see, e.g. \cite{V05}.

One of the most fundamental results in the field was presented in the seminal paper by Johnson and Lindenstrauss~\cite{JL84}.
\begin{lemma}[Distributional JL Lemma] \label{l:JL}
For every $n \in \mathbb{N}$ and $\varepsilon, \delta \in (0,1)$, there exists a random $m \times n$ projection matrix $A$, where $m = \Theta(\varepsilon^{-2}\lg\frac{1}{\delta})$ such that for every $x \in \mathbb{R}^n$
\begin{equation}
\Pr[ \; | \;\|Ax\|_2^2 - \|x\|_2^2\; | < \varepsilon \|x\|_2^2 \;] \ge 1 - \delta
\label{eq:probBound}
\end{equation}
\end{lemma}

\noindent The target dimension $m$ in the lemma is known to be optimal \cite{JW13, LN17}. 

\paragraph{Running Time Performances.} Perhaps the most common proof of the lemma (see, e.g. \cite{DG03, M08}) samples a projection matrix by independently sampling each entry from a standard Gaussian (or Rademacher) distribution.
Such matrices are by nature very dense, and thus a na\"ive embedding runs in $O(m\|x\|_0)$ time, where $\|x\|_0$ is the number of non-zero entries of $x$.

Due to the algorithmic significance of the lemma, much effort was invested in finding techniques to accelerate the embedding time.
One fruitful approach for accomplishing this goal is to consider a distribution over {\em sparse} projection matrices.
This line of work was initiated by Achlioptas \cite{A03}, who constructed a distribution over matrices, in which the {\em expected} fraction of non-zero entries is at most one third, while maintaining the target dimension.
The best result to date in constructing a sparse Johnson-Lindenstrauss matrix is due to Kane and Nelson \cite{KN14}, who presented a distribution over matrices satisfying~\eqref{eq:probBound} in which every column has at most $s=O(\varepsilon^{-1}\lg(1/\delta))$ non-zero entries. Conversely Nelson and Nguy$\tilde{\mbox{\^{e}}}$n \cite{NN13} showed that this is almost asymptotically optimal. That is, every distribution over $n \times m$ matrices satisfying~\eqref{eq:probBound} with $m=\Theta(\varepsilon^{-2}\lg(1/\delta))$, and such that every column has at most $s$ non-zero entries must satisfy $s = \Omega((\varepsilon\lg(1/\varepsilon))^{-1}\lg(1/\delta))$.

While the bound presented by Nelson and Nguy$\tilde{\mbox{\^{e}}}$n is theoretically tight, we can provably still do much better in practice. 
Specifically, the lower bound is attained on vectors $x \in \mathbb{R}^n$ for which, loosely speaking, the "mass" of $x$ is concentrated in few entries. Formally, the ratio $\|x\|_\infty/\|x\|_2$ is large. However, in practical scenarios, such as the term frequency - inverse document frequency representation of a document, we may often assume that the mass of $x$ is "well-distributed" over many entries (That is, $\|x\|_\infty/\|x\|_2$ is small). In these common scenarios projection matrices which are significantly sparser turn out to be very effective.

\paragraph{Feature Hashing.} 
In the pursuit for sparse projection matrices, Weinberger \etal \newline \cite{WKD+09} introduced dimensionality reduction via {\em Feature Hashing}, in which the projection matrix $A$ is, in a sense, as sparse as possible. That is, every column of $A$ contains exactly one non-zero entry, randomly chosen from $\{-1,1\}$. 
This technique is one of the most influential mathematical tools in the study of scaling-up machine learning algorithms, mainly due to its simplicity and good performance in practice \cite{D13, S15}.
More formally, for $n,m \in \mathbb{N}^+$, the projection matrix $A$ is sampled as follows.
Sample $h \in_R [n] \to [m]$, and $\sigma = \left\langle\sigma_j\right\rangle_{j \in [n]} \in_R \{-1,1\}^n$ independently.
For every $i \in [m], j \in [n]$, let $a_{ij} = a_{ij}(h,\sigma) \eqdef \sigma_j \cdot \mathbbm{1}_{h(j)=i}$ (that is, $a_{ij}=\sigma_j$ iff $h(j)=i$ and $0$ otherwise). 
Weinberger \etal additionally showed exponential tail bounds on $\|Ax\|_2^2$ when the ratio $\|x\|_{\infty} / \|x\|_2$ is sufficiently small, and $m$ is sufficiently large. These bounds were later improved by Dasgupta \etal \cite{DKS10} and most recently by Dahlgaard, Knudsen and Thorup \cite{DKT17} improved these concentration bounds.
Conversely, a result by Kane and Nelson \cite{KN14} implies that if we allow $\|x\|_{\infty} / \|x\|_2$ to be too large, then there exist vectors for which~\eqref{eq:probBound} does not holds. 

Finding the correct tradeoffs between $\|x\|_{\infty} / \|x\|_2$, and $m,\varepsilon, \delta$ in which feature hashing performs well remained an open problem.
Our main contribution is settling this problem, and providing a complete and comprehensive understanding of the performance of feature hashing. 

\subsection{Main results}
The main result of this paper is a tight tradeoff between the target dimension $m$, the approximation ratio $\varepsilon$, the error probability $\delta$ and $\|x\|_\infty/\|x\|_2$. More formally, let $\varepsilon, \delta > 0$ and $m \in \mathbb{N}^+$. Let $\nu(m,\varepsilon, \delta)$ be the maximum $\nu \in [0,1]$ such that for every $x \in \mathbb{R}^n$, if $\|x\|_{\infty} \le \nu \|x\|_2$ then \eqref{eq:probBound} holds.
Our main result is the following theorem, which gives tight asymptotic bounds for the performance of feature hashing, thus closing the long-standing gap.

\begin{theorem} \label{th:main}
There exist constants $C \ge D>0$ such that for every $\varepsilon, \delta \in (0,1)$ and $m \in \mathbb{N}^+$ the following holds. If $\frac{C \lg \frac{1}{\delta}}{\varepsilon^2} \le m  < \frac{2}{\varepsilon^2 \delta}$ then 
\begin{equation*}\nu(m,\varepsilon, \delta) = \Theta\left( \sqrt{\varepsilon} \cdot \min\left\{ \frac{\lg\frac{\varepsilon m}{\lg \frac{1}{\delta}}}{\lg\frac{1}{\delta}}, \sqrt{\frac{\lg \frac{\varepsilon^2 m}{\lg \frac{1}{\delta}}}{\lg\frac{1}{\delta}}} \right\} \right) \;.\end{equation*}
Otherwise, if $m \ge \frac{2}{\varepsilon^2 \delta}$ then $\nu(m,\varepsilon, \delta) = 1$. Moreover if $m < \frac{D\lg\frac{1}{\delta}}{\varepsilon^2}$ then $\nu(m, \varepsilon, \delta)=0$.
\end{theorem}

While the bound presented in the theorem may strike as surprising, due to the intricacy of the expressions involved, the tightness of the result shows that this is, in fact, the correct and "true" bound. Moreover, the proof of the theorem demonstrates how both branches in the $\min$ expression are required in order to give a tight bound. 

\paragraph{Experimental Results.} Our theoretical bounds are accompanied by empirical results that shed light on the nature of the constants in Theorem~\ref{th:main}. Our empirical results show that in practice the constants inside the Theta-notation are significantly tighter than the theoretical proof might suggest, and in fact feature hashing performs well for a larger scope of vectors. Specifically, our main result implies that whenever $\frac{4 \lg \frac{1}{\delta}}{\varepsilon^2} \le m  < \frac{2}{\varepsilon^2 \delta}$, \begin{equation*}\nu(m,\varepsilon, \delta) \ge 0.725 \sqrt{\varepsilon} \cdot \min\left\{ \frac{\lg\frac{\varepsilon m}{\lg \frac{1}{\delta}}}{\lg\frac{1}{\delta}}, \sqrt{\frac{\lg \frac{\varepsilon^2 m}{\lg \frac{1}{\delta}}}{\lg\frac{1}{\delta}}} \right\} \;,\end{equation*}
(except for very sparse vectors, i.e. $\|x\|_0 \le 7$) whereas the theoretical proof provides a smaller constant $2^{-6}$ in front of $\sqrt{\varepsilon}$. Since feature hashing satisfies~\eqref{eq:probBound} whenever $\|x\|_\infty \le \nu(m,\varepsilon, \delta)\|x\|_2$, this implies that feature hashing works well on even a larger range of vectors than the theory suggests.

\paragraph{Proof Technique}
As a fundamental step in the proof of Theorem~\ref{th:main} we prove tight asymptotic bounds for high-order norms of the approximation factor.\footnote{Given a random variable $X$ and $r>0$, the $r$th norm of $X$ (if exists) is defined as $\|X\|_r \eqdef\sqrt[r]{\mathbb{E}(|X|^r)}$.}
More formally, for every $x \in \mathbb{R}^n \setminus \{0\}$ let $X(x) = |\|Ax\|_2^2 - \|x\|_2^2|$. The technical crux of our results is tight bounds on high-order moments of $X(x)$. Note that by rescaling we may restrict our focus without loss of generality to unit vectors. 
\begin{notation}
For every $m,r,k > 0$ denote
\begin{equation*}\Lambda(m,r,k) = \left\{\begin{array}{lr}
        \sqrt{\frac{r}{m}}, & k \ge mr \\
        \max\left\{\sqrt{\frac{r}{m}}, \frac{r^2}{k\ln^2\left(\frac{emr}{k}\right)}\right\}, & mr > k \ge \sqrt{mr}\\
        \max\left\{\sqrt{\frac{r}{m}}, \frac{r^2}{k\ln^2\left(\frac{emr}{k}\right)}, \frac{r}{k\ln\left(\frac{emr}{k^2}\right)}\right\}, & \sqrt{mr} > k
        \end{array}\right. \;.\end{equation*}
\end{notation}								
In these notations our main technical lemmas are the following. 
\begin{lemma} \label{l:upperBound}
For every even $r \le m/4$ and unit vector $x \in \mathbb{R}^n$, $\|X(x)\|_r = O(\Lambda(m,r,\|x\|_\infty^{-2}))$.
\end{lemma}

\begin{lemma} \label{l:lowerBound}
For every $k \le n$ and even $r \le \min\{m/4, k\}$, $\|X(\xk)\|_r = \Omega\left(\Lambda(m,r,k)\right)$, where $\xk \in \mathbb{R}^n$ is the unit vector whose first $k$ entries equal $\tfrac{1}{\sqrt{k}}$.
\end{lemma}

While it might seem at a glance that bounding the high-order moments of $X(x)$ is merely a technical issue, known tools and techniques could not be used to prove Lemmas~\ref{l:upperBound}, \ref{l:lowerBound}. Particularly, earlier work by Kane and Nelson \cite{KN14, CJN18} and Freksen and Larsen \cite{FL17} used high-order moments bounds as a step in proving probability tail bounds of random variables. The existing techniques, however, can not be adopted to bound high-order moments of $X(x)$ (see also Section~\ref{sec:related}), and novel approaches were needed. Specifically, our proof incorporates a novel combinatorial scheme for counting edge-labeled Eulerian graphs.

\paragraph{Previous Results.} Weinberger \etal \cite{WKD+09} showed that if $m = \Omega(\varepsilon^{-2}\lg(1/\delta))$, then $\nu(m,\varepsilon, \delta) = \Omega(\varepsilon \cdot (\lg(1/\delta)\lg(m/\delta))^{-1/2})$. Dasgupta \etal \cite{DKS10} showed that under similar conditions $\nu(m,\varepsilon, \delta) = \Omega(\sqrt{\varepsilon} \cdot (\lg(1/\delta)\lg^2(m/\delta))^{-1/2})$. These bounds were recently improved by Dahlgaard \etal \cite{DKT17} who showed that $\nu(m,\varepsilon, \delta) = \Omega\left( \sqrt{\varepsilon} \cdot \sqrt{\frac{\lg(1/\varepsilon)}{\lg(1/\delta)\lg(m/\delta)}}\right)$. Conversely, Kane and Nelson \cite{KN14} showed that for the restricted case of $m = \Theta(\varepsilon^{-2}\lg(1/\delta))$, $\nu(m,\varepsilon, \delta) = O\left(\sqrt{\varepsilon}  \cdot \tfrac{\lg(1/\varepsilon)}{\lg(1/\delta)}\right)$, which matches the bound in Theorem~\ref{th:main} if, in addition, $\lg(1/\varepsilon) \le \sqrt{\lg(1/\delta)}$. 

\paragraph{Key Tool : Counting Labeled Eulerian Graphs.} 
Our proof presents a new combinatorial result concerning Eulerian graphs. 
Loosely speaking, we give asymptotic bounds for the number of labeled Eulerian graphs containing a predetermined number of nodes and edges. 
Formally, let $\alpha, \beta, r$ be integers such that $1 \le \beta \le \alpha/2 \le \min\{n/2,r/2\}$. 
Let ${\cal G}_{\alpha,\beta,r}$ denote the family of all edge-labeled Eulerian multigraphs $G=([\alpha], E_G , \pi_G)$, such that 
\begin{enumerate}
	\item $G$ has no isolated vertices;
	\item $|E_G|=r$, and $\pi_G : E_G \to [r]$ is a bijection, which assigns a label in $[r]$ to each edge; and
	\item the number of connected components in $G$ is $\beta$.
\end{enumerate}
\begin{notation}
Denote $\Delta = \Delta(\alpha, \beta) \eqdef \alpha^{2 \alpha} \beta^{- \beta} \left[(\alpha - 2\beta)^2 + 4 (\alpha - \beta)\right]^{r - \alpha}$.
\end{notation}

\begin{theorem} \label{th:main2}
$2^{-O(r)} \cdot \Delta(\alpha, \beta) \le |{\cal G}_{\alpha,\beta,r}| \le 2^{O(r)} \cdot \Delta(\alpha, \beta)$. 
\end{theorem}
\subsection{Related Work}\label{sec:related}

The CountSketch scheme, presented by Charikar \etal \cite{CCF04}, was shown to satisfy~\eqref{eq:probBound} by Thorup and Zhang \cite{TZ12}. The scheme essentially samples $O(\lg(1/\delta))$ independent copies of a feature hashing matrix with $m=O(\varepsilon^{-2})$ rows, and applies them all to $x$. The estimator for $\|x\|_2^2$ is then given by computing the median norm over all projected vectors. The CountSketch scheme thus constructs a sketching matrix $A$ such that every column has $O(\lg(1/\delta))$ non-zero entries. However, this construction does not provide a norm-preserving embedding into a Euclidean space (that is, the estimator of $\|x\|_2^2$ cannot be represented as a norm of $Ax$), which is essential for some applications such as nearest-neighbor search \cite{HIM12}.

Kane and Nelson~\cite{KN14} presented a simple construction for the so-called sparse Johnson Lindenstrauss transform. This is a distribution of $m \times n$ matrices, for $m = \Theta(\varepsilon^{-2}\lg(1/\delta))$, where every column has $s$ non-zero entries, randomly chosen from $\{-1,1\}$. Note that if $s=1$, this distribution yields the feature hashing one. Kane and Nelson showed that for $s=\Theta(\varepsilon m)$ this construction satisfies~\eqref{eq:probBound}. Recently, Cohen \etal \cite{CJN18} presented two simple proofs for this result. While their proof methods give (simple) bounds for high-order moments similar to those in Lemmas~\ref{l:upperBound} and ~\ref{l:lowerBound}, they rely heavily on the fact that $s$ is relatively large. Specifically, for $s=1$ the bounds their method or an extension thereof give are trivial.


\section{Counting Labeled Eulerian Graphs}

In this section we prove Theorem~\ref{th:main2}. In order to upper bound $|{\cal G}_{\alpha,\beta,r}|$, we give an encoding scheme and show that every graph $G \in {\cal G}_{\alpha,\beta,r}$ can be encoded in a succinct manner, thus bounding $|{\cal G}_{\alpha,\beta,r}|$.  

\paragraph{Encoding Argument.} Fix a graph $G \in {\cal G}_{\alpha,\beta,r}$, and let $\left\langle \{j_p,\ell_p\} \right\rangle_{p \in [r]}$ be its ordered sequence of edges. 
In what follows, we give an encoding algorithm that, given $G$, produces a "short" bit-string ${\cal E}$ that encodes $G$. 
The string ${\cal E}$ is a concatenation of three strings ${\cal E}_T, {\cal E}_{Eu}, {\cal E}_R$, encoded as follows.

Let ${\cal CC}(G) = \{C_1,\ldots,C_\beta\}$ be the set of connected components of $G$ ordered by the smallest labeled node in each component, and for every $j \in [\beta]$, denote the graph induced by $C_j$ in $G$ by $G[C_j] = (C_j, E_j)$. 
For every $j \in [\beta]$ the encoding algorithm chooses a set $T_j \subseteq E_j$ of edges of a spanning tree in $C_j$. Denote by $E_T$ the union of all trees in $G$. 
\begin{proposition}
$|E_T| = \alpha - \beta$.
\end{proposition}
\begin{proof}
For every $j \in [\beta]$, $|E_j| = |C_j| - 1$. Therefore $|E_T| = \sum_{j \in [\beta]}{|E_j|} = \alpha - \beta$.
\end{proof}
Let $e_1,\ldots,e_{\alpha - \beta}$ be the ordering of $E_T$ induced by $\pi_G$. The algorithm encodes ${\cal E}_T$ to be the list of $\alpha - \beta$ edges in $\binom{V}{2}$, followed by an encoding of $\pi_G(E_T)$ as a set in $\binom{[r]}{\alpha - \beta}$. 
Next, since every connected component is Eulerian, for every $j \in [\beta]$, there is an edge $e_j \in E_j \setminus E_T$. Let $E_{Eu}$ denote the set of all $\beta$ such edges, and let $e_{j_1},\ldots,e_{j_\beta}$ be the ordering of $E_{Eu}$ induced by $\pi_G$. For every $i \in [\beta]$, the algorithm encodes a pair $(j_i,(x_i,y_i)) \in [\beta] \times \binom{C_{j_i}}{2}$, and appends them in order together with $\pi_G(E_{Eu})$ to encode ${\cal E}_{Eu}$.
Finally, the algorithm encodes $E_G \setminus (E_T \cup E_{Eu})$ in the ordering induced by $\pi_G$ as a list of length $r - \alpha$ in $\bigcup_{j \in [\beta]}{\binom{C_j}{2}}$. Denote this list of the rest of the edges by ${\cal E}_R$.

\begin{lemma} \label{l:encodeG}
${\cal E}$ can be encoded using at most $\lg \Delta(\alpha, \beta) + O(r)$ bits.
\end{lemma}

\begin{proof}
In order to bound the length of ${\cal E}$ we shall bound each of the three strings separately.
One can encode an ordered list of $\alpha - \beta$ distinct unordered pairs in $V$ using at most $(\alpha - \beta)\lg \binom{\alpha}{2}$ bits. Therefore ${\cal E}_T$ can be encoded using at most 
\begin{equation}
(\alpha - \beta)\lg \binom{\alpha}{2} + \lg \binom{r}{\alpha - \beta} \le 2(\alpha - \beta)\lg\alpha + r 
\label{eq:boundTrees}
\end{equation}
bits. 

Next, for every $i \in [\beta]$, $(j_i,(x_i,y_i))$ can be encoded using $\lg \beta\binom{|C_{j_i}|}{2}$ bits. Therefore ${\cal E}_{Eu}$ can be encoded using at most 
\begin{equation}
\sum_{i \in [\beta]}{\lg \beta\binom{|C_{j_i}|}{2}} + \lg \binom{r - \alpha}{\beta} \le \beta \lg \beta + 2\lg\prod_{i \in [\beta]}|C_{j_i}| + r \le \beta \lg \beta + 2\beta\lg \frac{\alpha}{\beta} + r
\label{eq:boundExtraEdge}
\end{equation}
bits, where the last inequality follows from the AM-GM inequality, since $\sum_{i \in \beta}{|C_{j_i}|} = \alpha$.

Finally, note that ${\cal E}_{R}$ can be encoded using $(r - \alpha) \lg \left(\sum_{j \in \beta}{\binom{C_j}{2}}\right) \le (r - \alpha) \lg \left(\sum_{j \in \beta}{|C_j|^2}\right)$ bits. Since 
\begin{equation*}\max\left\{\sum_{j \in [\beta]}{x_j^2} : \sum_{j \in [\beta]}{x_j} = \alpha \ge 2 \beta \; and \; \forall j \in [\beta]. \; x_j \ge 2\right\} = (\alpha - 2 (\beta - 1))^2 + 4 (\beta - 1) \;,\end{equation*}
we get that ${\cal E}_{R}$ can be encoded using 
\begin{equation}
(r - \alpha) \lg \left[(\alpha - 2 (\beta - 1))^2 + 4 (\beta - 1)\right] = (r - \alpha) \lg \left[(\alpha - 2\beta)^2 + 4 (\alpha - \beta)\right]
\label{eq:boundRest}
\end{equation}
bits.
Summing over \eqref{eq:boundTrees}, \eqref{eq:boundExtraEdge} and \eqref{eq:boundRest} implies the lemma.
\end{proof}

\begin{lemma}\label{l:decodeG}
Given ${\cal E}$, one can reconstruct $G$.
\end{lemma}
\begin{proof}
In order to prove the lemma, we give a decoding algorithm that receives ${\cal E}$ and constructs $G$.
The algorithm first reads the first list of $\alpha - \beta$ elements of $\binom{V}{2}$ from ${\cal E}_T$, followed by $\pi_G(E_T)$, to decode $E_T$, and the restriction $\pi_G|_{_{E_T}}$ of $\pi_G$ to $E_T$.
Given the set of spanning trees, the algorithm constructs ${\cal CC}(G) = \{C_1,\ldots,C_\beta\}$ (note that the ordering on ${\cal CC}(G)$ is inherent in the components themselves, and does not depend on $\pi_G$).
Next, the algorithm reads ${\cal E}_{Eu}$ and recovers the set $E_{Eu}$ of edges, along with the restriction $\pi_G|_{_{E_{Eu}}}$ of $\pi_G$ to $E_{Eu}$.
Finally, the algorithm reads ${\cal E}_{R}$ and reconstructs the remaining $r - \alpha$ edges, with their induced ordering. Since $\pi_G(E_G \setminus (E_T \cup E_{Eu})) = [r] \setminus \pi_G(E_T \cup E_{Eu})$, the algorithm can reconstruct the restriction $\pi_G|_{_{E_{R}}}$ of $\pi_G$ to $E_{R}$, thus reconstructing $\pi_G$.

\end{proof}

\begin{corollary}
$|{\cal G}_{\alpha,\beta,r}| \le 2^{O(r)} \Delta(\alpha, \beta)$.
\end{corollary}

Next we turn to lower bound $|{\cal G}_{\alpha,\beta,r}|$. To this end, we construct a subset of $|{\cal G}_{\alpha,\beta,r}|$ of size at least $2^{-O(r)}\Delta(\alpha, \beta)$, thus lower bounding $|{\cal G}_{\alpha,\beta,r}|$. 

Consider the following family ${\cal H}_{\alpha,\beta,r}$ of labeled multigraphs over the vertex set $[\alpha]$. For every $H = ([\alpha],E_H,\pi_H) \in {\cal H}_{\alpha,\beta,r}$, $H$ contains $\beta$ connected components, where $\beta - 1$ components, referred to as {\em small} are composed of $2$ vertices each, and one {\em large} component contains the remaining $\alpha - 2(\beta - 1)$ nodes. The first $\alpha$ edges (according to $\pi_H$) are a union of $\beta$ simple cycles, where each cycle contains the entire set of nodes of one connected component.
\begin{claim}
$|{\cal H}_{\alpha,\beta,r}| \ge 2^{-O(r)}\Delta(\alpha, \beta)$.
\end{claim}

\begin{proof}
The number of possible ways to choose the partition of $[\alpha]$ into $\beta$ connected components such that all but one contain exactly $2$ vertices is $\frac{1}{\beta!}\binom{\alpha}{2,2,\ldots,2,\alpha - 2(\beta - 1)}$. Each small component has exactly one spanning cycle, while the large component has $(\alpha - 2(\beta - 1) - 1)!$ spanning cycles. Once the cycles are chosen, there are at least $2^{- \beta}\alpha!$ ways to order the edges.
The number of possible edges in $H$ is $(\beta - 1) \cdot \binom{2}{2} + \binom{\alpha - 2(\beta - 1)}{2}$. Therefore there are $\left[\binom{\alpha - 2(\beta - 1)}{2} + (\beta - 1)\right]^{r - \alpha}$ ways to choose the ordered sequence of $r - \alpha$ remaining edges. We conclude that
\begin{equation*}
\begin{split}
|{\cal H}_{\alpha,\beta,r}| &\ge \frac{1}{\beta!}\binom{\alpha}{2,\ldots,2,\alpha - 2(\beta - 1)}(\alpha - 2\beta + 1)!2^{- \beta}\alpha!\left[\binom{\alpha - 2(\beta - 1)}{2} + 2(\beta - 1)\right]^{r - \alpha} \\
&\ge 2^{-O(r)}\beta^{- \beta} \cdot (\alpha!)^2 \cdot \left[(\alpha - 2\beta + 2)^2 + 4(\beta - 1)\right]^{r - \alpha} \ge 2^{-O(r)}\Delta(\alpha, \beta) \;.
\end{split}
\end{equation*}
\end{proof}

\begin{lemma}
$\Pr_{H \in_R {\cal H}_{\alpha,\beta,r}}[H \in_R {\cal G}_{\alpha,\beta,r}] \ge 2^{-O(r)}$.
\end{lemma}

\begin{proof}
Every $H \in {\cal H}_{\alpha,\beta,r}$ contains $r$ labeled edges, no isolated vertices and $\beta$ connected components. Therefore, $H \in {\cal G}_{\alpha,\beta,r}$ if and only if the degree of every node in $H$ is even.
Let $E_{\lambda}$ be the set the last $r - \alpha$ edges in $H$, then since $E \setminus E_{\lambda}$ is a union of disjoint cycles spanning all vertices, for every $v \in V$, $deg_{E_H}(v) = deg_{E_H \setminus E_{\lambda}}(v)+deg_{E_{\lambda}}(v) = 2 + deg_{E_{\lambda}}(v)$. 
Hence $H \in {\cal G}_{\alpha,\beta,r}$ if and only if for every $v \in V$, $deg_{E_{\lambda}}(v)$ is even.
Consider the set of all possible $\left[(\alpha - 2\beta + 2)^2 + 4(\beta - 1)\right]^{(r - \alpha)/2}$ sequences of $(r - \alpha)/2$ edges in $H$. For every such sequence $s$, let the signature of $s$ be the indicator vector $\sigma(s) \in \{0,1\}^V$, where for every $v \in V$, $\sigma(s)_v = 1$ if and only if $deg_s(v)$ is odd. Let $s_1,s_2$ be of the same signature, and let $E_{\lambda}$ be the edge sequence of length $r - \alpha$, which is the concatenation of $s_1$ and $s_2$. Then $deg_{E_{\lambda}}(v)$ is even. Since the number of possible signatures is $2^{\alpha}$, there exists a set $S$ of edge sequences of length $(r - \alpha)/2$ that all have the same signature such that $|S| \ge 2^{-\alpha}\left[(\alpha - 2\beta + 2)^2 + 4(\beta - 1)\right]^{(r - \alpha)/2}$.
Therefore
\begin{equation*}\Pr_{H \in_R {\cal H}_{\alpha,\beta,r}}[H \in_R {\cal G}_{\alpha,\beta,r}] \ge \Pr[E_{\lambda} \in S \times S ] \ge 2^{-O(r)} \;.\end{equation*}
\end{proof}
We therefore conclude the following, which finishes the proof of Theorem~\ref{th:main2}.
\begin{corollary}
$|{\cal G}_{\alpha,\beta,r}| \ge 2^{-O(r)}\Delta(\alpha, \beta)$.
\end{corollary}

\ifpdf
\section{Bounding \texorpdfstring{$\nu(m,\varepsilon, \delta)$}{k}}
\else
\section{Bounding $\nu(m,\varepsilon, \delta)$}
\fi

In this section we prove Theorem~\ref{th:main}, assuming Lemmas~\ref{l:upperBound} and \ref{l:lowerBound}, whose proof is deferred to section~\ref{sec:moment}.
Fix $\varepsilon, \delta \in (0,1)$ and an integer $m$.
We first address the case where $m \ge \frac{2}{\varepsilon^2 \delta}$. Let $x \in \mathbb{R}^n$ be a unit vector. Then
\begin{equation*}
\begin{split}
\mathbb{E}\left[\left| \|Ax\|_2^2 - 1 \right|^2\right] &= \mathbb{E}\left[\left(\sum_{j \ne \ell \in [n]}{ \mathbbm{1}_{h(j)=h(\ell)} \cdot \sigma_j \sigma_\ell\cdot x_jx_\ell}\right)^2\right] \\
&= \mathbb{E}\left[2\sum_{j \ne \ell \in [n]}{ \mathbbm{1}_{h(j)=h(\ell)}  x_j^2x_\ell^2}\right] \le \frac{2}{m}
\end{split}
\end{equation*}
Therefore by Chebyshev's inequality $\Pr\left[ \; \left| \|Ax\|_2^2 - 1 \right| \ge \varepsilon \; \right] \le \delta$.

We therefore continue assuming $m < \frac{2}{\varepsilon^2 \delta}$. 
From Lemmas~\ref{l:upperBound} and \ref{l:lowerBound} there exist $C_1,C_2 > 0$ such that for every $r,k$, if $r \le m/4$ then for every unit vector $x$, $\|X(x)\|_r \le 2^{C_2} \Lambda(m,r,k)$. Moreover, if $r \le k$ then
\begin{equation*}2^{-C_1} \Lambda(m,r,k) \le \|X(\xk)\|_r \le 2^{C_2} \Lambda(m,r,k) \;.\end{equation*}
Note that in addition $\Lambda(m,2r,k) \le 4 \Lambda(m,r,k)$. Denote $\hat{C}=2^{C_2+2}$, and $C = 2C_1 + 2C_2+5$. 

\begin{lemma}
If $m < \frac{\log \frac{1}{\delta}}{4C\varepsilon^2}$ then $\nu(m,\varepsilon, \delta) = 0$.
\end{lemma}

\begin{proof}
Let $r = \frac{1}{C}\lg \frac{1}{\delta}$, and let $k \ge 2r$ be some integer.
Then 
\begin{equation*}\mathbb{E}\left[\left(X(\xk)\right)\right] = \|X(\xk)\|_r^r \ge \left(\frac{r}{m}\right)^{r/2} \ge 2 \varepsilon^r \;.\end{equation*}
Applying the Paley-Zygmund inequality
\begin{equation}
\begin{split}
\Pr\left[ \; \left| \|A\xk\|_2^2 - 1 \right| > \varepsilon \;\right] &= \Pr\left[ \; \left|A\xk - 1 \right|^r > \varepsilon^r \;\right] \\
&\ge \Pr\left[ \left(X(\xk)\right)^r > 2^{-1} \mathbb{E}[X(\xk)^r] \;\right] \\
&\ge \frac{\mathbb{E}^2[X(\xk)^r]}{4\mathbb{E}[X(\xk)^{2r}]} = \frac{\|X(\xk)\|_r^{2r}}{4\|X(\xk)\|_{2r}^{2r}} \ge \frac{1}{4}\left(\frac{2^{-C_1}{\Lambda(m,r,k)}}{2^{C_2}\Lambda(m,2r,k)}\right)^{2r} \\
&\ge \frac{1}{4}\left(\frac{2^{-C_1}}{2^{C_2+2}}\right)^{2r} = 2^{-(2C_1-2C_2-4)r-2} \ge \delta
\label{eq:PZ}
\end{split}
\end{equation}
Therefore $\nu(m, \varepsilon, \delta) \le \|\xk\|_\infty = \frac{1}{\sqrt{k}}$ for every $k \ge 2r$, which implies $\nu(m, \varepsilon, \delta)=0$.
\end{proof}
For the rest of the proof we assume that $\frac{\hat{C} \log \frac{1}{\delta}}{\varepsilon^2} \le m < \frac{2}{\varepsilon^2\delta}$, and we start by proving a lower bound on $\nu$.

\begin{lemma}
$\nu(m,\varepsilon, \delta) = \Omega\left(\min\left\{\frac{\sqrt{\varepsilon}}{\lg\frac{1}{\delta}}\lg\frac{\varepsilon m}{\lg \frac{1}{\delta}},  \sqrt{\frac{\varepsilon\lg\frac{\varepsilon^2 m}{\lg \frac{1}{\delta}}}{\lg\frac{1}{\delta}}}\right\}\right)$.
\end{lemma}

\begin{proof}
Let $r = \lg \frac{1}{\delta}$, let $x \in \mathbb{R}^n$ be a unit vector such that $\|x\|_\infty  \le \min\left\{\frac{\sqrt{\varepsilon} \ln \frac{e\varepsilon m}{r}}{\sqrt{2^{C_2}e} r}, \sqrt{\frac{\varepsilon\lg\frac{e \varepsilon^2 m}{r}}{2^{C_2}er}}\right\}$, and let $k \eqdef \frac{1}{\|x\|_\infty^2} \ge \max\left\{\frac{2^{C_2}e r^2}{\varepsilon \ln^2 \frac{e\varepsilon m}{r}}, \frac{2^{C_2}er}{\varepsilon\lg \frac{e \varepsilon^2 m}{r}}\right\}$. If $k \le mr$, then since $\frac{r^2}{k \ln^2 \frac{emr}{k}}$ is convex as a function of $k \in \left[\frac{2^{C_2}e r^2}{\varepsilon \ln^2 \frac{e\varepsilon m}{r}},mr\right]$ then
\begin{equation*}2^{C_2}\frac{r^2}{k \ln^2 \frac{emr}{k}} \le \max\left\{ \frac{r}{m} ,\left(\frac{\varepsilon \ln^2 \frac{e \varepsilon m}{r}}{e\ln^2 \frac{ e \varepsilon m \ln^2 \frac{e\varepsilon m}{r}}{r}}\right)\right\} < \varepsilon/2\;.\end{equation*}
Moreover, if $k \le \sqrt{mr}$ then since $\frac{r}{k \ln \frac{emr}{k^2}}$ is convex as a function of $k \in \left[\frac{2^{C_2}e r}{\varepsilon \ln \frac{e\varepsilon^2 m}{r}} , \sqrt{mr}\right]$, then 
\begin{equation*}2^{C_2}\frac{r}{k \ln \frac{emr}{k^2}} \le \max\left\{ \sqrt{\frac{r}{m}}, \left(\frac{\varepsilon \ln \frac{e \varepsilon^2 m}{r}}{e \ln \frac{\varepsilon^2 m \ln^2\frac{e \varepsilon^2 m}{r}}{r}}\right) \right\}  \le \varepsilon/2\; .\end{equation*}
Since clearly, $\sqrt{\frac{2^{2C_2}r}{m}} \le \varepsilon/2$, then by Lemma~\ref{l:upperBound} we have $\|X(x)\|_r^r \le (\varepsilon/2)^r$, and thus 
\begin{equation}
\begin{split}
\Pr\left[ \; \left| \|Ax\|_2^2 - 1 \right| > \varepsilon \;\right] &= \Pr\left[ \; \left| \|Ax\|_2^2 - 1 \right|^r > \varepsilon^r \;\right] \\
&\le \Pr\left[ \left(X(x)\right)^r > 2^r \mathbb{E}[X(x)^r] \;\right] \le 2^{-r} = \delta \;.
\label{eq:Markov}
\end{split}
\end{equation}
Hence $\nu(m,\varepsilon, \delta) \ge \min\left\{\frac{\sqrt{\varepsilon} \ln \frac{e\varepsilon m}{r}}{\sqrt{2^{C_2}e} r}, \sqrt{\frac{\varepsilon\lg\frac{e \varepsilon^2 m}{r}}{2^{C_2}er}}\right\} = \Omega\left( \min\left\{\frac{\sqrt{\varepsilon}}{\lg\frac{1}{\delta}}\lg \frac{\varepsilon m}{\lg \frac{1}{\delta}}, \sqrt{\frac{\varepsilon\lg\frac{\varepsilon^2 m}{\lg \frac{1}{\delta}}}{\lg\frac{1}{\delta}}}\right\} \right)$.
\end{proof}

\begin{lemma}\label{l:upperBoundNu}
$\nu(m,\varepsilon, \delta) = O\left(\min\left\{  \frac{\sqrt{\varepsilon}}{\lg\frac{1}{\delta}}\lg\frac{\varepsilon m}{\lg \frac{1}{\delta}}, \sqrt{\frac{\varepsilon\lg\frac{\varepsilon^2 m}{\lg \frac{1}{\delta}}}{\lg\frac{1}{\delta}}}\right\}\right)$.
\end{lemma}
To this end, let $r = \frac{1}{C}\lg \frac{1}{\delta}$, and denote \begin{equation*}t = \min\left\{  \frac{\sqrt{e\varepsilon}}{r}\ln\frac{e\varepsilon m}{r}, \sqrt{\frac{e\varepsilon\ln\frac{e\varepsilon^2 m}{r}}{r}}\right\} = O\left(\min\left\{  \frac{\sqrt{\varepsilon}}{\lg\frac{1}{\delta}}\lg\frac{\varepsilon m}{\lg \frac{1}{\delta}}, \sqrt{\frac{\varepsilon\lg\frac{\varepsilon^2 m}{\lg \frac{1}{\delta}}}{\lg\frac{1}{\delta}}}\right\}\right) \;.\end{equation*} 

Assume first that $t \le \frac{1}{\sqrt{r}}$, and let $k = \frac{1}{t^2}$. We will show that $\mathbb{E}\left[\left(X(\xk)\right)^r \right] \ge 2\varepsilon^r$. Since $t \le \frac{1}{\sqrt{r}}$, then $k \ge r$. If $\frac{\sqrt{e\varepsilon}}{r}\ln\frac{e\varepsilon m}{r} \le \sqrt{\frac{e\varepsilon\ln\frac{e\varepsilon^2 m}{r}}{r}}$, then $k = \frac{r^2}{e \ln^2 \frac{e\varepsilon m}{r}}$. Since $\frac{e \varepsilon m }{r} > e$, then $k \le mr$. Therefore
\begin{equation*}\mathbb{E}\left[\left(X(\xk)\right)^r \right] = \|X(\xk)\|_r^r \ge \left(\frac{r^2}{k \ln^2 \frac{emr}{k}}\right)^r = \left(\frac{e\varepsilon \ln^2 \frac{e \varepsilon m}{r}}{\ln^2 \frac{ e^2 \varepsilon m \ln^2 \frac{e\varepsilon m}{r}}{r}}\right)^r \ge 2\varepsilon^r\;.\end{equation*}
Otherwise, $k = \frac{r}{e \varepsilon \ln \frac{e \varepsilon^2 m}{r}}$. Moreover, since $\frac{\varepsilon^2 m }{r} > 1$, then $k \le r / \varepsilon \le \sqrt{mr}$. Therefore
\begin{equation*}\mathbb{E}\left[\left(X(\xk)\right)^r \right] = \|X(\xk)\|_r^r \ge \left(\frac{r}{k\ln \frac{emr}{k^2}}\right)^r = \left(\frac{e \varepsilon \ln \frac{e \varepsilon^2 m}{r}}{\ln \frac{e^3 \varepsilon^2 m \ln^2 \frac{e \varepsilon ^2 m}{r}}{r}}\right)^r  \ge 2\varepsilon^r\;.\end{equation*}
Applying the Paley-Zygmund inequality we get that 
similarly to \eqref{eq:PZ} 
\begin{equation*}
\Pr\left[ \; \left| \|A\xk\|_2^2 - 1 \right| > \varepsilon \;\right] \ge \Pr\left[ \left(X(\xk)\right)^r > 2^{-1} \mathbb{E}[X(\xk)^r] \;\right] \ge \delta
\end{equation*}
Therefore $\nu(m, \varepsilon, \delta) \le \|\xk\|_\infty = t$.

Assume next that $\frac{1}{\sqrt{r}} < t < \sqrt{\frac{\varepsilon}{4}}$, and note that since $\frac{\sqrt{e\varepsilon}}{r}\ln\frac{e\varepsilon m}{r} \ge \frac{1}{\sqrt{r}}$, then $m > e^{\sqrt{r/(e \varepsilon)}}$, and since $\sqrt{\frac{e\varepsilon\ln\frac{e\varepsilon^2 m}{r}}{r}} \ge \frac{1}{\sqrt{r}}$ then $m > e^{1/(e \varepsilon)}$.
Let $k = \frac{1}{t^2}$, and consider independent $h \in_R [n] \to [m]$, and $\sigma = (\sigma_1, \ldots, \sigma_m) \in_R \{-1,1\}^m$. Let $y \in \mathbb{R}^{n}$ be defined as follows. For every $j \in [n]$, $y_j = \xk_j$ if and only if $h(j)=1$, and $y_j = 0$ otherwise. Denote $z = \xk - y$. Then $\|\xk\|_2^2 = \|y\|_2^2 + \|z\|_2^2$, and moreover, $\|A\xk\|_2^2 = \|Ay\|_2^2 + \|Az\|_2^2$, where $A = A(h, \sigma)$. Let ${\cal E}_{first}$ denote the event that $|h^{-1}(\{1\})| = 2 \sqrt{\varepsilon k}$, and that for all $j \in [n]$, if $h(j)=1$ then $\sigma_j = 1$, and let ${\cal E}_{rest}$ denote the event that $\left|\|Az\|_2^2 - \|z\|_2^2\right| < \varepsilon\|z\|_2^2$. By Chebyshev's inequality, $\Pr[{\cal E}_{rest} \mid {\cal E}_{first}] = \Omega(1)$.
Note that if $k = \frac{r^2}{e\varepsilon \ln^2 \frac{e \varepsilon m}{r}}$, then 
\begin{equation*}
\begin{split}
2\sqrt{\varepsilon k} &= \frac{r}{\sqrt{e}\ln \frac{e \varepsilon m}{r}} \le \frac{\lg \frac{1}{\delta}}{C\sqrt{e}\ln \frac{e \varepsilon m}{r}} \le \frac{\lg \frac{1}{\delta}}{C\sqrt{e}(\ln em - \ln\frac{1}{\varepsilon} - \ln r)} \\
&\le \frac{\lg \frac{1}{\delta}}{C\sqrt{e}(\ln m - 3\ln \ln m)} \le \frac{\lg \frac{1}{\delta}}{2\ln m} \;,
\end{split}
\end{equation*}
where the inequality before last is due to the fact that $m > \max\{e^{1/e\varepsilon}, e^{\sqrt{r}}\}$,
and otherwise, $k = \frac{r}{e \varepsilon \ln \frac{e \varepsilon^2 m}{r}}$, and 
\begin{equation*}2 \sqrt{\varepsilon k} \le \varepsilon k = \frac{\lg \frac{1}{\delta}}{eC\ln \frac{e \varepsilon^2 m}{r}} = \frac{\lg \frac{1}{\delta}}{eC(\ln em - 2\ln\frac{1}{\varepsilon} - \ln r)} \le \frac{\lg \frac{1}{\delta}}{eC(\ln em - 4\ln \ln m)} \le \frac{\lg \frac{1}{\delta}}{2\ln m} \;.\end{equation*}
Therefore for small enough $\varepsilon$,
\begin{equation*}
\begin{split}
\Pr[{\cal E}_{first}] &= \binom{k}{2\sqrt{\varepsilon k}}\cdot \left(\frac{1}{m}\right)^{2\sqrt{\varepsilon k}} \cdot \left(1 - \frac{1}{m}\right)^{k - 2\sqrt{\varepsilon k}} \cdot 2^{-2\sqrt{\varepsilon k}} \\
&\ge \left(\frac{1}{m}\right)^{2\sqrt{\varepsilon k}} \cdot \left(1 - \frac{1}{m}\right)^{r} \cdot 2^{-r} \ge \left(\frac{1}{m}\right)^{2\sqrt{\varepsilon k}} \cdot 2^{-\frac{2}{C}\lg\frac{1}{\delta}} \ge \delta^{3/4} \;.\\
\end{split}
\end{equation*}
We conclude that for small enough $\delta$, $\Pr[{\cal E}_{first} \wedge {\cal E}_{rest}] \ge \delta$. Conditioned on ${\cal E}_{first} \wedge {\cal E}_{rest}$ we get that 
\begin{equation*}
\begin{split}
\|A\xk\|_2^2 &= \|Ay\|_2^2 + \|Az\|_2 \ge \frac{4 \varepsilon k }{k} + (1 - \varepsilon)\|z\|_2^2 \\
& =4 \varepsilon + (1 - \varepsilon)\cdot \frac{k - 2\sqrt{\varepsilon k}}{k} \ge 4 \varepsilon + (1 - \varepsilon)^2 > 1 + \varepsilon \;,
\end{split}
\end{equation*}
where the inequality before last is due to the fact that $k \ge \frac{4}{\varepsilon}$. Therefore $\nu(m,\varepsilon, \delta) \le \|\xk\|_\infty = t$.

Finally, assume $t > \sqrt{\frac{\varepsilon}{4}}$. Since $\sqrt{\frac{e \varepsilon \ln \frac{e \varepsilon^2 m}{r}}{r}} \ge t > \sqrt{\frac{\varepsilon}{4}}$, we get that $m \ge \frac{r}{e \varepsilon^2}e^{r/(4e)} \ge \frac{r}{e \varepsilon^2 \delta^{1/(4eC)}}$.
Let $k = \frac{2}{\varepsilon}$.
Consider independent $h \in_R [n] \to [m]$, and $\sigma = (\sigma_1, \ldots, \sigma_m) \in_R \{-1,1\}^m$, and let $A = A(h,\sigma)$. Let ${\cal E}_{col}$ denote the event that there are $j \ne \ell \in [k]$ such that for every $p \ne q \in [k]$, $h(p)=h(q)$ if and only if $\{p,q\} = \{j,\ell\}$.  Then for small enough $\varepsilon, \delta$,
\begin{equation*}
\begin{split}
\Pr[{\cal E}_{col}] &= \binom{k}{2} \cdot \frac{1}{m} \cdot \prod_{j \in [k-1]}{\left(1 - \frac{j}{m}\right)}
\ge \frac{k^2}{2m}\cdot (1 - \varepsilon/2) \cdot \left(1 - \frac{k}{m}\right)^k\\
&\ge \frac{k^2}{2m}\cdot (1 - \varepsilon/2) \cdot \left(1 - \frac{k^2}{m}\right) \ge 2 \delta\cdot (1 - \varepsilon/2) \cdot \left(1 - 4e \delta^{1/(4Ce)}\right) \ge \delta \;.\\
\end{split}
\end{equation*}
Conditioned on ${\cal E}_{col}$ we get that 
$\left| \|A\xk\|_2^2 - 1 \right| = \frac{2}{k} = \varepsilon$.
Therefore $\nu(m,\varepsilon, \delta) \le \sqrt{\frac{\varepsilon}{2}} \le O(t)$.

This completes the proof of Lemma~\ref{l:upperBoundNu}, and thus of 
Theorem~\ref{th:main}.

\ifpdf
\section{Bounding the Moments of \texorpdfstring{$| \|Ax\|_2- 1 |$}{Bounding Moments}} \label{sec:moment}
\else
\section{Bounding the Moments of $| \|Ax\|_2-1 |$} \label{sec:moment}
\fi
In this section we prove Lemmas~\ref{l:upperBound} and \ref{l:lowerBound}.
Recall that $h \in_R [n] \to [m]$, and $\sigma_1,\ldots,\sigma_m \in_R \{1,-1\}$ are independent.
For every $i \in [m], j \in [n]$, $a_{ij} \eqdef \sigma_j \cdot \mathbbm{1}_{h(j)=i}$.
For every unit vector $x \in \mathbb{R}^n \setminus \{0\}$ we let $X = X(x) = |\|Ax\|_2^2 - 1|$.
\noindent We start with providing a better understanding of $X$.

\begin{equation*}
\|Ax\|_2^2 = 1 + \sum_{j \ne \ell \in [n]}{\mathbbm{1}_{h(j)=h(\ell)} \cdot \sigma_j \sigma_\ell\cdot x_jx_\ell}
\end{equation*}
Denote $I_{[n]} = \{(p,p) : p \in [n]\}$. Then
\begin{equation*}
X = \left| \|Ax\|_2^2 - 1 \right| = \left| \sum_{(j,\ell) \in ([n]\times [n]\setminus I_{[n]})}{ \mathbbm{1}_{h(j)=h(\ell)} \cdot \sigma_j \sigma_\ell\cdot x_jx_\ell}\right|  \;,
\end{equation*}
and therefore for every even $r$
\begin{equation}
\begin{split}
\|X\|_r^r=\mathbb{E}[X^r] &= \sum_{\left\langle (j_p,\ell_p) \right\rangle_{p \in [r]} \in ([n]\times[n] \setminus I_{[n]})^r}\mathbb{E}\left[\prod_{p \in [r]}{\mathbbm{1}_{h(j_p)=h(\ell_p)}\sigma_{j_p} \sigma_{\ell_p}x_{j_p} x_{\ell_p}}\right]
\end{split}
\label{eq:sumOfExpectations}
\end{equation}
Every $S = \left\langle (j_p,\ell_p) \right\rangle_{p \in [r]} \in ([n]\times[n] \setminus I_{[n]})^r$ defines a directed multigraph $\overrightarrow{G_S}$ with $r$ ordered directed edges on vertex set $[n]$. Let $G_S$ denote the underlying undirected multigraph.
\begin{definition}
Let $q \in [n]$. The {\em degree of $q$ in $S$} is the degree of $q$ in $G_S$. Namely $d_S(q) \eqdef \left|\{p \in [r] : q \in \{j_p,\ell_p\} \} \right|$. 
\end{definition}

\begin{notation}
Given $S \in ([n]\times[n] \setminus I_{[n]})^r$, let ${\cal CC}(S)$ denote the set of all connected components of $G_S$ that contain at least two nodes. Let $\beta(S) \eqdef |{\cal CC}(S)|$, $V(S) = \bigcup_{C \in {\cal CC}(S)}C$ and $\alpha(S) \eqdef \left| V(S)\right|$.
\end{notation}

Next, for every integer $\beta$ and a subset $V \subseteq [n]$, let ${\cal S}_{V,\beta} \subseteq ([n]\times[n] \setminus I_{[n]})^r$ be the set of all sequences $S \in ([n]\times[n] \setminus I_{[n]})^r$ such that 
\begin{enumerate}
	\item For every $q \in [n]$, $d_S(q)$ is even; and
	\item $V(S)= V$ and $\beta(S) = \beta$.
\end{enumerate}

\begin{lemma} \label{l:sum}
\begin{equation}
\|X\|_r^r \le \|x\|_{\infty}^{2r} \sum\limits_{\beta=1}^{r/2}{\sum\limits_{\alpha = 2 \beta}^{r}{\frac{m^{\beta}}{\left(\|x\|_{\infty}^{2}m\right)^{\alpha}}\sum\limits_{V \in \binom{[n]}{\alpha}}{|{\cal S}_{V,\beta}| \cdot \prod_{q \in V}{x_q^2}}}} \;,
\label{eq:lemmaSum}
\end{equation}
Moreover, if for all $j \in \supp(x)$, $|x_j| = \|x\|_{\infty}$, then equality holds.
\end{lemma}
\begin{proof}
Fix some $S = \left\langle (j_p,\ell_p) \right\rangle_{p \in [r]} \in ([n]\times[n] \setminus I_{[n]})^r$.
Then
\begin{equation}
\begin{split}
\mathbb{E}&\left[\prod_{p \in [r]}{\mathbbm{1}_{h(j_p)=h(\ell_p)}\sigma_{j_p} \sigma_{\ell_p}x_{j_p} x_{\ell_p}}\right] = \mathbb{E}\left[\prod_{p \in [r]}{\mathbbm{1}_{h(j_p)=h(\ell_p)}} \cdot \prod_{q \in V(S)}{\sigma_q^{d_S(q)}} \cdot \prod_{q \in V(S)}{x_q^{d_S(q)}}\right]\\
&=\prod_{q \in V(S)}{x_q^{d_S(q)}} \cdot \mathbb{E}\left[\prod_{p \in [r]}{\mathbbm{1}_{h(j_p)=h(\ell_p)}}\right] \cdot \prod_{q \in V(S)}{\mathbb{E}\left[\sigma_q^{d_S(q)}\right]} \;
\end{split}
\label{eq:prod}
\end{equation}
where the last equality follows from independence.
Assume first that for some $q \in V(S)$, $d_S(q)$ is odd. Then $\mathbb{E}\left[\sigma_q^{d_S(q)}\right] = 0$, and therefore \eqref{eq:prod} equals $0$.
Otherwise, $\mathbb{E}\left[\sigma_q^{d_S(q)}\right] = 1$ for all $q \in V(S)$. We therefore assume hereafter that $d_S(q)$ is even for all $q \in V(S)$.
For every $C \in {\cal CC}(S)$, $C$ contains an edge of $G_S$, thus there exists $p \in [r]$ such that $j_p, \ell_p \in C$. Conversely, for every $p \in [r]$ there exists a unique connected component $C \in {\cal CC}(S)$ such that $j_p, \ell_p \in C$. Therefore 
\begin{equation*}\mathbb{E}\left[\prod_{p \in [r]}{\mathbbm{1}_{h(j_p)=h(\ell_p)}}\right] = \mathbb{E}\left[\prod_{C \in {\cal CC}(S)}{\prod_{p \in [r] : j_p \in C}{\mathbbm{1}_{h(j_p)=h(\ell_p)}}}\right] = \prod_{C \in {\cal CC}(S)}{\mathbb{E}\left[\prod_{p \in [r] : j_p \in C}{\mathbbm{1}_{h(j_p)=h(\ell_p)}}\right]} \;,\end{equation*}
where the last equality is due to independence.
Next, let $C = \{v_1, \ldots, v_{|C|}\} \in {\cal CC}(S)$, then $\mathbb{E}\left[\prod_{p \in [r] : j_p \in C}{\mathbbm{1}_{h(j_p)=h(\ell_p)}}\right] = \mathbb{E}\left[\mathbbm{1}_{h(v_1) = \ldots = h(v_{|C|})}\right] = \frac{1}{m^{|C|-1}}$. We thus conclude that 
\begin{equation*}\prod_{C \in {\cal CC}(S)}{\mathbb{E}\left[\prod_{p \in [r] : j_p \in C}{\mathbbm{1}_{h(j_p)=h(\ell_p)}}\right]} = \prod_{C \in {\cal CC}(S)}{\frac{1}{m^{|C|-1}}} = \frac{1}{m^{\alpha(S)-\beta(S)}} \;.\end{equation*}
For every sequence $S$ that donates a non-zero summand to the sum, since $d_S(q)$ is even for all $q \in V(S)$ every $C \in {\cal CC}(S)$ is Eulerian, and therefore contains at least two nodes and two edges. 
Therefore $1 \le \beta(S) \le r/2$ and $2 \beta(S) \le \alpha(S) \le r$. Plugging this into \eqref{eq:sumOfExpectations} we get that
\begin{equation}
\begin{split}
\|X\|_r^r &= \sum_{\substack{S \in ([n]\times[n] \setminus I_{[n]})^r \\ \forall q \in V(S). \; d_S(q) \in \mathbb{N}_{even}}}{\frac{1}{m^{\alpha(S) - \beta(S)}}\prod_{q \in V(S)}{x_q^{d_S(q)}}} \\
&= \sum_{\beta=1}^{r/2}{\sum_{\alpha=2 \beta}^{r}{\sum_{V \in \binom{[n]}{\alpha}}{\sum_{S \in {\cal S}_{V,\beta}}{\frac{1}{m^{\alpha - \beta}}\prod_{q \in V}{x_q^{d_S(q)}}}}}}
\end{split}
\label{eq:equality}
\end{equation}
For every $q \in V(S)$, $d_S(q)$ is a positive even integer, and therefore $d_S(q)-2 \ge 0$ is also even. Hence for every $q \in V(S)$, $x_q^{d_S(q)-2} = |x_q|^{d_S(q)-2} \le \|x\|_{\infty}^{d_S(q)-2}$. Since $\sum_{q \in V(S)}{d_S(q)}=2r$, then $\prod_{q \in V}{x_q^{d_S(q)}} \le \|x\|_\infty^{2r-2\alpha}\prod_{q \in V}{x_q^2}$. Moreover, equality holds if for all $j \in \supp(x)$, $|x_j| = \|x\|_{\infty}$. Plugging this in \eqref{eq:equality} we get that 
\begin{equation*}
\begin{split}
\|X\|_r^r &\le \|x\|_{\infty}^{2r} \sum\limits_{\beta=1}^{r/2}{\sum\limits_{\alpha = 2 \beta}^{r}{\frac{m^{\beta}}{\left(\|x\|_{\infty}^{2}m\right)^{\alpha}}\sum\limits_{V \in \binom{[n]}{\alpha}}{\sum\limits_{S \in {\cal S}_{V, \beta}}{\prod_{q \in V}{x_q^2}}}}} \\
&= \|x\|_{\infty}^{2r} \sum\limits_{\beta=1}^{r/2}{\sum\limits_{\alpha = 2 \beta}^{r}{\frac{m^{\beta}}{\left(\|x\|_{\infty}^{2}m\right)^{\alpha}}\sum\limits_{V \in \binom{[n]}{\alpha}}{|{\cal S}_{V,\beta}| \cdot \prod_{q \in V}{x_q^2}}}}
\end{split}
\end{equation*}
\end{proof}

\ifpdf
\subsection{Upper Bounding \texorpdfstring{$\|X\|_r$}{Xrr}}
\else
\subsection{Upper Bounding $\|X\|_r$}
\fi
We start by proving Lemma~\ref{l:upperBound}. To this end, denote $k = \|x\|_{\infty}^{-2}$, and for every $1 \le \beta \le \alpha/2 \le r/2$, denote 
\begin{equation*}
\begin{split}
M(\alpha, \beta) &= \left(m\beta^{-1}\right)^{\beta} \left(\frac{k\alpha}{m}\right)^{\alpha}(\alpha - 2\beta)^{2r - 2 \alpha} \\
N(\alpha, \beta) &= \left(m\beta^{-1}\right)^{\beta} \left(\frac{k\alpha}{m}\right)^{\alpha}(\alpha - \beta)^{r - \alpha}\;. 
\end{split}
\end{equation*}
Applying Theorem~\ref{th:main2} to the expression in Lemma~\ref{l:sum} we can conclude the following.
\begin{claim} \label{c:upperBoundPart1}
$\|X\|^r_r \le \frac{2^{O(r)}}{k^r} \sum\limits_{\beta=1}^{r/2}{\sum\limits_{\alpha = 2 \beta}^{r}{(M(\alpha, \beta) + N(\alpha, \beta))}}$.
\end{claim}

\begin{proof}
Let $1 \le \beta \le \alpha/2 \le r/2$. Then for every $V \in \binom{[n]}{\alpha}$, every sequence in ${\cal S}_{V,\beta}$ defines a directed edge-labeled multigraph $\overrightarrow{G_S}$ on $V$, whose underlying undirected graph $G_S$ is Eulerian and has $\beta$ connected components. Invoking the notation used in Theorem~\ref{th:main2}, $G_S$ is isomorphic to some graph in ${\cal G}_{\alpha,\beta,r}$, and moreover, it defines at most $2^r$ sequences in ${\cal S}_{V,\beta}$. Thus $|{\cal S}_{V,\beta}| \le 2^r|{\cal G}_{\alpha,\beta,r}|\le 2^{O(r)}\Delta(\alpha, \beta)$. Plugging this in the \eqref{eq:lemmaSum} we get that
\begin{equation}
\|X\|_r^r \le \frac{2^{O(r)}}{k^r} \sum\limits_{\beta=1}^{r/2}{\sum\limits_{\alpha = 2 \beta}^{r}{m^{\beta}\left(\frac{k}{m}\right)^{\alpha}\Delta(\alpha, \beta)\sum\limits_{V \in \binom{[n]}{\alpha}}{\prod_{q \in V}{x_q^2}}}} \;.
\label{eq:upperBound1}
\end{equation}
For every $V \in \binom{[n]}{\alpha}$, the coefficient of $\prod_{q \in V}{x_q^2}$ in the expansion of $\left(\sum_{q \in [n]}{x_q^2}\right)^{\alpha}$ is $\alpha!$. Therefore 
\begin{equation*}1 = \left(\sum_{q \in [n]}{x_q^2}\right)^{\alpha} \ge \alpha! \sum_{V \in \binom{[n]}{\alpha}}{\prod_{q \in V}{x_q^2}}\;.\end{equation*}
Plugging this in \eqref{eq:upperBound1} we get
\begin{equation*}
\begin{split}
\|X\|_r^r &\le \frac{2^{O(r)}}{k^r} \sum\limits_{\beta=1}^{r/2}{\sum\limits_{\alpha = 2 \beta}^{r}{m^{\beta} \left(\frac{k}{m}\right)^{\alpha}\frac{\Delta(\alpha, \beta)}{\alpha!}}} \\
&\le \frac{2^{O(r)}}{k^r} \sum\limits_{\beta=1}^{r/2}{\sum\limits_{\alpha = 2 \beta}^{r}{\left(m\beta^{-1}\right)^{\beta}\left(\frac{k\alpha}{m}\right)^{\alpha}\left[(\alpha - 2\beta)^{2} + 4(\alpha - \beta)\right]^{r - \alpha}}}\\
&\le \frac{2^{O(r)}}{k^r}\sum\limits_{\beta=1}^{r/2}{\sum\limits_{\alpha = 2 \beta}^{r}{(M(\alpha, \beta) + N(\alpha, \beta))}}\;.
\end{split}
\end{equation*}
\end{proof}

\begin{lemma} \label{l:boundM}
For all $1 \le \beta \le \alpha/2 \le r/2$, then if $k \ge mr$, then $M(\alpha, \beta) \le 2^{O(r)}\left(\frac{k^2 r}{m}\right)^{r/2}$. Otherwise, $M(\alpha, \beta) \le 2^{O(r)} \max\left\{\left(\frac{r}{\ln \frac{2emr}{k}}\right)^{2r}, \left(\frac{k^2 r}{m}\right)^{r/2}\right\}$.
\end{lemma}

\begin{proof}
Let $1 \le \beta \le \alpha/2 \le r/2$.
First note that 
\begin{equation*}\alpha^\alpha \le 2^{O(r)} \alpha! = 2^{O(r)} (\alpha - 2 \beta)! \cdot (2 \beta)! \cdot \binom{\alpha}{2 \beta} \le 2^{O(r)}(\alpha - 2\beta)^{\alpha - 2 \beta} \beta^{2 \beta} \;,\end{equation*}
and therefore $M(\alpha, \beta) \le 2^{O(r)}(m\beta)^\beta \cdot \left(\frac{k}{m}\right)^{\alpha}(\alpha-2\beta)^{2r - \alpha - 2 \beta}$.

Next, we fix some $\beta \in [r/2]$. Define $f, \hat{f} : (2 \beta, + \infty) \to \mathbb{R}$ by $f(\alpha) = \left(\frac{k}{m}\right)^{\alpha} \cdot (\alpha - 2 \beta)^{2r - \alpha - 2 \beta}$, and $\hat{f}(\alpha) = \ln \frac{k}{m} - \ln(\alpha - 2 \beta) + \frac{2r - 4\beta}{\alpha - 2 \beta} - 1$ for all $\alpha > 2 \beta$. Then $f'(\alpha) = f(\alpha) \cdot \hat{f}(\alpha)$, and $\hat{f}'(\alpha) = - \frac{1}{\alpha - 2 \beta} - \frac{2(r - 2 \beta)}{(\alpha - 2 \beta)^2} < 0$ for all $\alpha > 2 \beta$. 

Assume first that $\beta \ge \frac{r}{2} - \frac{ek}{2m} $, then
$\hat{f}(r)  = \ln \frac{k}{m} - \ln(r - 2 \beta) + 1 \ge 0$. Therefore $\hat{f}(\alpha) \ge 0$, thus $f'(\alpha) \ge 0$ and $f(\alpha) \le f(r)$ for all $\alpha \in (2 \beta, r]$. It follows that 
\begin{equation}
M(\alpha, \beta) \le 2^{O(r)}\left(m\beta\right)^{\beta}\left(\frac{k}{m}\right)^{r}(r - 2 \beta)^{r - 2\beta} \le 2^{O(r)} \left(\frac{m}{r/2}\right)^{r/2}\left(\frac{kr}{m}\right)^{r} \le 2^{O(r)} \left(\frac{k^2r}{m}\right)^{r/2}\;, 
\label{eq:boundLargeBeta}
\end{equation} 
where the inequality before last follows from the fact that $m \ge r$ and $\beta \le r/2$. 

To prove the first part of the lemma, note that if $k \ge \frac{mr}{e}$, then for all $\beta \in (0,r/2]$, $\beta \ge \frac{r}{2} - \frac{ek}{2m}$. 

To prove the second part of the lemma, we next assume that $k < \frac{mr}{e}$ and $\beta < \frac{r}{2} - \frac{ek}{2m}$, and note that this implies a tighter bound  on $k$, namely $k < \frac{1}{e}m(r-2 \beta)$. 
Let
\begin{equation*}\alpha_0 = 2 \beta + \frac{2(r - 2 \beta)}{\ln \frac{2em(r - 2 \beta)}{k}} \quad \quad, \quad \quad \alpha_1 = 2 \beta + \frac{2(r - 2 \beta) }{\ln \frac{2em(r - 2 \beta)}{k\ln \frac{2em(r - 2 \beta)}{k}}} \;. \end{equation*}
Then $2 \beta \le \alpha_0 < \alpha_1$, and moreover

\begin{equation*}
\hat{f}(\alpha_0) = \ln \frac{k}{m} - \ln\frac{2(r - 2 \beta)}{\ln \frac{2em(r - 2 \beta)}{k}} + \frac{2(r - 2\beta)}{\frac{2(r - 2 \beta)}{\ln \frac{2em(r - 2 \beta)}{k}}} - 1 = \ln\ln \frac{2em(r - 2 \beta)}{k} > 0 \;,
\end{equation*}
and similarly

\begin{equation*}
\hat{f}(\alpha_1) =  \ln\ln \frac{2em(r - 2 \beta)}{k\ln \frac{2em(r - 2 \beta)}{k}} - \ln \ln \frac{2em(r - 2 \beta)}{k} < 0 \;.
\end{equation*}

Therefore there exists a unique $\alpha^* \in (\alpha_0, \alpha_1)$ such that $\hat{f}(\alpha^*)=0$, and thus $\frac{k}{em(\alpha^* - 2 \beta)} = e^{- \frac{2r - 4 \beta}{\alpha^* - 2 \beta}} \le 1$. Moreover, for every $\alpha > 2 \beta$,

\begin{equation*}
f(\alpha) \le f(\alpha^*) = \left(\frac{k}{m}\right)^{\alpha^*} \cdot (\alpha^* - 2 \beta)^{2r - \alpha^* - 2 \beta} \le \left(\frac{k}{m}\right)^{2 \beta}(\alpha^* - 2 \beta)^{2r - 4 \beta} \;,
\end{equation*}
and we get that 
\begin{equation}
\begin{split}
M(\alpha, \beta) &\le 2^{O(r)}(m \beta)^\beta \left(\frac{k}{m}\right)^{2 \beta}(\alpha^* - 2 \beta)^{2r - 4 \beta} \le 2^{O(r)}\left(\frac{k^2 \beta}{m}\right)^{\beta} (\alpha_1 - 2 \beta)^{2r - 4 \beta} \\
& \le 2^{O(r)}\left(\frac{k^2 \beta}{m}\right)^{\beta} \left(\frac{r}{\ln \frac{2emr}{k\ln \frac{2emr}{k}}}\right)^{2r - 4 \beta}
\label{eq:upperBound2}
\end{split}
\end{equation}
where the last inequality follows from the fact that $y = \frac{x}{\ln \frac{x}{\ln x}}$ is monotonically increasing for $x > 1$.

Finally, define $g, \hat{g} : (0, + \infty) \to \mathbb{R}$ by \begin{equation*}g(\beta) = \left(\frac{k^2 \beta}{m}\right)^{\beta} \left(\frac{r}{\ln \frac{2emr}{k\ln \frac{2emr}{k}}}\right)^{2r - 4 \beta} \; and \quad \hat{g}(\beta) = \ln \frac{k^2}{m} + \ln\beta + 1 -4 \ln \left(\frac{r}{\ln \frac{2emr}{k\ln \frac{2emr}{k}}}\right)\end{equation*} for all $\beta > 0$. Then $g'(\beta) = g(\beta) \cdot \hat{g}(\beta)$, and moreover $g''(\beta) = g(\beta) \cdot \hat{g}^2(\beta) + g(\beta)/\beta > 0$. 
We thus conclude that $g$ is convex as a function of $ \beta \in (0,r/2]$, and therefore $g(\beta) \le \max\{\lim\limits_{\beta \to 0}g(\beta), g(r/2)\}$ for all $\beta \in (0, r/2]$.
\begin{equation*}\lim\limits_{\beta \to 0}g(\beta) = \left(\frac{r}{\ln \frac{2emr}{k\ln \frac{2emr}{k}}}\right)^{2r} \quad, \; and \quad g(r/2) \le 2^{O(r)} \left(\frac{k^2 r}{m}\right)^{r/2}\;.\end{equation*}
Since $k < \frac{mr}{e}$, then $\frac{2emr}{k} > e^2$, and therefore $\ln \frac{2emr}{k\ln \frac{2emr}{k}} \ge \frac{1}{2}\ln \frac{2emr}{k}$. 
Plugging into \eqref{eq:upperBound2}, we thus get that since $\ln \frac{2emr}{k} > 1$,
\begin{equation*}
M(\alpha, \beta) \le 2^{O(r)} \max\left\{\left(\frac{r}{\ln \frac{2emr}{k}}\right)^{2r}, \left(\frac{k^2 r}{m}\right)^{r/2}\right\} \;,
\end{equation*}
and the proof of the lemma is now complete.
\end{proof}

\begin{lemma} \label{l:boundN}
Let $1 \le \beta \le \alpha/2 \le \ell/2$.  Then if $k^2>mr$, then $N(\alpha, \beta) \le \left(\frac{k^2r}{m}\right)^{r/2}$ and otherwise, $N(\alpha, \beta) \le \max\left\{M(\alpha, \beta), \left(\frac{r}{\ln \frac{emr}{k^2}}\right)^{r} \right\}$.
\end{lemma}

\begin{proof}
Assume first that $k^2 \ge mr$. 
If $k>m$, then $\left(\frac{k\alpha}{m}\right)^{\alpha}(\alpha - \beta)^{r - \alpha}$ is increasing as a function of $\alpha$ over $[2 \beta, r]$, and therefore,
\begin{equation}
N(\alpha, \beta) \le \left(\frac{kr}{m}\right)^{r} \cdot \left(\frac{m}{\beta}\right)^{\beta} \le 2^{O(r)}\left(\frac{k^2r}{m}\right)^{r/2}  \;.
\label{eq:kLarger}
\end{equation} 
Otherwise, since $2 \beta \le \alpha \le r$,
\begin{equation}
N(\alpha, \beta) \le \left(\frac{k^2}{\beta m}\right)^{\beta} \alpha^{\alpha}(\alpha - \beta)^{r - \alpha}  \le r^r\left(\frac{k^2}{\beta m}\right)^{\beta} \le 2^{O(r)}\left(\frac{k^2r}{m}\right)^{r/2} \;.
\label{eq:kMedium}
\end{equation}
Next, we assume that $k^2 < rm$, and note that whenever $\alpha > 4 \beta$, then $(\alpha - 2 \beta)^2 > (\alpha - \beta)$, and therefore $N(\alpha, \beta) \le M(\alpha, \beta)$. Otherwise, $\alpha^{\alpha}(\alpha - \beta)^{r - \alpha}  \le 2^{O(r)}\beta^{\alpha}\beta^{\alpha - \beta} \le 2^{O(r)}\beta^r$, and therefore
\begin{equation}
N(\alpha, \beta) \le \left(\frac{k^2}{m\beta}\right)^{\beta} \cdot \beta^r 
\label{eq:kSmaller}
\end{equation}
Define next $g, \hat{g} : (0, + \infty) \to \mathbb{R}$ by $g(\beta) = \beta^{r}\left(\frac{k^2}{\beta m}\right)^\beta$ and $\hat{g}(\beta) = \frac{r}{\beta} + \ln \frac{k^2}{\beta m} - 1$ for every $\beta >0$. Then $g'(\beta) = g(\beta) \hat{g}(\beta)$, $g(\beta) > 0$ and $\hat{g}'(\beta) = -\frac{r}{\beta^2} - \frac{1}{\beta} < 0$ for every $\beta > 0$. Let 
\begin{equation*}\beta_0 = \frac{r}{\ln \frac{emr}{k^2}} \quad \quad, \quad \quad \beta_1 = \frac{r}{\ln \frac{emr}{k^2\ln \frac{emr}{k^2}}} \;. \end{equation*}

Then $0 < \beta_0 < \beta_1 \le r$, and moreover

\begin{equation*}
\hat{g}(\beta_0) = \frac{r}{\frac{r}{\ln \frac{emr}{k^2}}} + \ln \frac{k^2}{\frac{r}{\ln \frac{emr}{k^2}} m} - 1 = \ln \ln \frac{emr}{k^2} \ge \ln \ln e = 0 \;,
\end{equation*}
where the last inequality is due to the fact that $k^2 < mr$. In addition,
\begin{equation*}
\hat{g}(\beta_1) = \frac{r}{\frac{r}{\ln \frac{emr}{k^2\ln \frac{emr}{k^2}}}} + \ln \frac{k^2}{\frac{r}{\ln \frac{emr}{k^2\ln \frac{emr}{k^2}}} m} - 1 =   - \ln \ln \frac{emr}{k^2} + \ln \ln \frac{emr}{k^2\ln \frac{emr}{k^2}} < 0
\end{equation*}

Therefore there exists a unique $\beta^* \in (\beta_0, \beta_1)$ such that 
$0 = \hat{g}(\beta^*)= \frac{r}{\beta^*} + \ln \frac{k^2}{\beta^* m} - 1$, which in turn implies $\frac{k^2}{\beta^* m} = e^{1-r/\beta^*}$. Moreover, and for all $\beta > 0$, 
\begin{equation*}g(\beta) \le g(\beta^*)= (\beta^*)^{r}\left(\frac{k^2}{\beta^* m}\right)^{\beta^*} \le  \left(\frac{r}{\ln \frac{emr}{k^2\ln \frac{emr}{k^2}}}\right)^r\;.\end{equation*}
Since $\frac{emr}{k^2} > e$, $\ln \ln \frac{emr}{k^2} < \frac{1}{2}\ln \frac{emr}{k^2}$,
and we have  
$N(\alpha, \beta) \le 2^{O(r)}\left(\frac{r}{\ln \frac{emr}{k^2}}\right)^r$.
\end{proof}

We now turn to prove Lemma~\ref{l:upperBound}.

\begin{proof}[Proof of Lemma~\ref{l:upperBound}]
Assume first that $k \ge mr$. Then by Claim~\ref{c:upperBoundPart1} and Lemmas~\ref{l:boundM}, \ref{l:boundN} we get that
\begin{equation*}
\begin{split}
\|X\|^r_r &\le \frac{2^{O(r)}}{k^r} \sum\limits_{\beta=1}^{r/2}{\sum\limits_{\alpha = 2 \beta}^{r}{(M(\alpha, \beta) + N(\alpha, \beta))}} \le \frac{2^{O(r)}}{k^r}\left(\frac{k^2r}{m}\right)^{r/2} \;,
\end{split}
\end{equation*}
and therefore $\|X\|_r = O\left(\sqrt{\frac{r}{m}}\right)$. 

Next, assume that $mr > k \ge \sqrt{mr}$. Once again by Claim~\ref{c:upperBoundPart1} and Lemmas~\ref{l:boundM}, \ref{l:boundN} we get that
\begin{equation}
\begin{split}
\|X\|^r_r &\le \frac{2^{O(r)}}{k^r} \sum\limits_{\beta=1}^{r/2}{\sum\limits_{\alpha = 2 \beta}^{r}{(M(\alpha, \beta) + N(\alpha, \beta))}} \\
&\le \frac{2^{O(r)}}{k^r} \sum\limits_{\beta=1}^{r/2}{\sum\limits_{\alpha = 2 \beta}^{r}{ \max\left\{\left(\frac{r}{\ln \frac{2emr}{k}}\right)^{2r}, \left(\frac{k^2 r}{m}\right)^{r/2}\right\} + \left(\frac{k^2r}{m}\right)^{r/2}}} \\
& \le \frac{2^{O(r)}}{k^r}\max\left\{\left(\frac{r}{\ln \frac{2emr}{k}}\right)^{2r}, \left(\frac{k^2 r}{m}\right)^{r/2}\right\} \;,
\end{split}
\end{equation}
and therefore $\|X\|_r = O\left(\max\left\{\frac{r^2}{k\ln^2 \frac{2emr}{k}}, \sqrt{\frac{r}{m}}\right\}\right)$. 

Finally, assume that $\sqrt{mr}>k$. Once again by Claim~\ref{c:upperBoundPart1} and Lemmas~\ref{l:boundM}, \ref{l:boundN} we get that
\begin{equation}
\begin{split}
\|X\|^r_r &\le \frac{2^{O(r)}}{k^r} \sum\limits_{\beta=1}^{r/2}{\sum\limits_{\alpha = 2 \beta}^{r}{(M(\alpha, \beta) + N(\alpha, \beta))}} \\
&\le \frac{2^{O(r)}}{k^r} \sum\limits_{\beta=1}^{r/2}{\sum\limits_{\alpha = 2 \beta}^{r}{ \max\left\{\left(\frac{r}{\ln \frac{2emr}{k}}\right)^{2r}, \left(\frac{k^2 r}{m}\right)^{r/2}\right\} + \max\left\{M(\alpha, \beta), \left(\frac{r}{\ln \frac{emr}{k^2}}\right)^{r} \right\}}} \\
& \le \frac{2^{O(r)}}{k^r}\max\left\{\left(\frac{r}{\ln \frac{2emr}{k}}\right)^{2r}, \left(\frac{k^2 r}{m}\right)^{r/2}\right\} \;,
\end{split}
\end{equation}
and therefore $\|X\|_r = O\left(\max\left\{\frac{r}{k\ln \frac{emr}{k^2}}, \frac{r^2}{k\ln^2 \frac{2emr}{k}}, \sqrt{\frac{r}{m}}\right\}\right)$. 

\end{proof}
\ifpdf
\subsection{Lower Bounding \texorpdfstring{$\|X(\xk)\|_r$}{Xrr}}
\else
\subsection{Lower Bounding $\|X(\xk)\|_r$}
\fi
We finish this section by proving Lemma~\ref{l:lowerBound}. To this end, let $k \le n$, and recall that by Lemma~\ref{l:sum}, since for every $j \in \supp(\xk)$, $|\xk_j| = \|\xk\|_\infty$, then 
\begin{equation*}\|X(\xk)\|_r^r = \|\xk\|_{\infty}^{2r} \sum\limits_{\beta=1}^{r/2}{\sum\limits_{\alpha = 2 \beta}^{r}{\frac{m^{\beta}}{\left(\|\xk\|_{\infty}^{2}m\right)^{\alpha}}\sum\limits_{V \in \binom{[n]}{\alpha}}{|{\cal S}_{V,\beta}| \cdot \prod_{q \in V}{(\xk_q)^2}}}} \;.\end{equation*}
For every $V \subseteq [n]$, if $V \subseteq [k]$, then $\prod_{q \in V}{(\xk_q)^2} = \|\xk\|_\infty^{2|V|}$, and otherwise $\prod_{q \in V}{(\xk_q)^2} = 0$. Substituting $\|\xk\|_\infty = \frac{1}{\sqrt{k}}$, and applying Theorem~\ref{th:main2} we get that since $r \le k$ then
\begin{equation}\begin{split}
\|X(\xk)\|_r^r &= \frac{1}{k^r} \sum\limits_{\beta=1}^{r/2}{\sum\limits_{\alpha = 2 \beta}^{r}{\frac{m^{\beta}k^{\alpha}}{m^{\alpha}}\sum\limits_{V \in \binom{[k]}{\alpha}}{|{\cal S}_{V,\beta}| \cdot \frac{1}{k^{\alpha}}}}} \\
&\ge \frac{1}{k^r} \sum\limits_{\beta=1}^{r/2}{\sum\limits_{\alpha = 2 \beta}^{r}{\frac{m^{\beta}}{m^{\alpha}}\sum\limits_{V \in \binom{[k]}{\alpha}}{ 2^{-O(r)}\alpha^{2 \alpha}\beta^{-\beta} \left[(\alpha - 2 \beta)^2 +4(\alpha - \beta)\right]^{r - \alpha}}}} \\
&= \frac{2^{-O(r)}}{k^r} \sum\limits_{\beta=1}^{r/2}{\sum\limits_{\alpha = 2 \beta}^{r}{\frac{m^{\beta}}{m^{\alpha}}\cdot \binom{k}{\alpha}\cdot \alpha^{2 \alpha}\beta^{-\beta} \left[(\alpha - 2 \beta)^2 +4(\alpha - \beta)\right]^{r - \alpha}}} \\
&\ge \frac{2^{-O(r)}}{k^r} \sum\limits_{\beta=1}^{r/2}{\sum\limits_{\alpha = 2 \beta}^{r}{\left(m \beta^{-1}\right)^\beta\left(\frac{k \alpha}{m}\right)^{\alpha}\left[(\alpha - 2\beta)^{2r - 2 \alpha} + (\alpha - \beta)^{r-\alpha}\right]}}\;.
\label{eq:sumLowerBound}
\end{split}
\end{equation}
Setting $\alpha = r, \beta=r/2$ we get that 
\begin{equation*} \|X(\xk)\|_r \ge \sqrt[r]{\frac{2^{-O(r)}}{k^r} \left(m r^{-1}\right)^{r/2}\left(\frac{k r}{m}\right)^{r}}  = \Omega\left(\sqrt{\frac{r}{m}}\right) \;.\end{equation*}
Assume next that $k \le mr$ and let $\alpha = 2 + \frac{r}{\ln \left(\frac{emr}{k}\right)}, \beta=1$. Then $\left(\frac{k}{m (\alpha-2)}\right)^{\alpha-2} \ge 2^{-O(r)}$, and therefore 
\begin{equation}\begin{split}
\|X(\xk)\|_r^r &\ge \frac{2^{-O(r)}}{k^r} \cdot m \cdot \left(\frac{k \alpha}{m}\right)^{\alpha}(\alpha - 2)^{2r - 2 \alpha}\\
&\ge \frac{2^{-O(r)}}{k^r} \cdot (\alpha - 2)^{2r} \cdot \frac{k^2}{m} \cdot \left(\frac{k}{m}\right)^{\alpha - 2} \cdot (\alpha - 2)^{\alpha} \cdot (\alpha - 2)^{ - 2 \alpha} \\
&\ge \frac{2^{-O(r)}}{k^r} \cdot (\alpha - 2)^{2r} \cdot \frac{k^2}{m} \cdot \left(\frac{k}{m(\alpha - 2)}\right)^{\alpha - 2} \ge \left(\frac{2^{-O(1)}r^2}{k\ln^2 \left(\frac{emr}{k}\right)}\right)^r \;.
\label{eq:sumLowerBoundMiddlek}
\end{split}
\end{equation}
We conclude that
\begin{equation*} \|X(\xk)\|_r = \Omega\left(\frac{r^2}{k\ln^2 \left(\frac{emr}{k}\right)}\right) \;.\end{equation*}
Finally, assume that $k \le \sqrt{mr}$ and let $\alpha = \frac{2r}{\ln \frac{emr}{k^2}}, \beta= \frac{r}{\ln \frac{emr}{k^2}}$. Then $\left(\frac{k^2}{m \beta}\right)^{\beta} \ge 2^{-O(r)}$ and therefore
\begin{equation*} \|X(\xk)\|_r \ge \sqrt[r]{\frac{2^{-O(r)}}{k^r} \left(\frac{k^2}{m\beta}\right)^{\beta}\beta^{r}} = \Omega\left(\frac{r}{k\ln \left(\frac{emr}{k^2}\right)}\right) \;,\end{equation*}
and the proof of Lemma~\ref{l:lowerBound} is now complete.


\section{Empirical Analysis}
\label{sec:experiments}

The goal of the experiments is to give bounds on some of the constants
hidden in the main theorem. From our experiments we conclude that for
$\frac{4 \lg \frac{1}{\delta}}{\varepsilon^2} \leq m <
\frac{2}{\varepsilon^2 \delta}$ the constant inside the
$\Theta$-notation in \autoref{th:main} is at least $0.725$ except for
very sparse vectors ($\|x\|_0 \leq 7$), where the constant is at least
$0.6$. Furthermore, we confirm that $\nu(m, \varepsilon, \delta) = 1$
when $m \geq \frac{2}{\varepsilon^2 \delta}$ and that there exists
data points where $\nu(m, \varepsilon, \delta) < 1$ while
$m = \frac{2 - \gamma}{\varepsilon^2 \delta}$, for some small
$\gamma$.

\subsection{Experiment Setup and Analysis}
\label{sec:experiment-setup}

To arrive at the results, we ran experiments and analyzed the data in
several phases. In the first phase we varied the target dimension $m$
over exponentially spaced values in the range $[2^6, 2^{12}]$, and a
parameter $k$ which controls the ratio between the $\ell_\infty$ and
the $\ell_2$ norm. The values of $k$ varied over exponentially spaced values in
the range $[2^1, 2^{13}]$. Then for all $m$ and $k$, we generated
$2^{24}$ vectors $x$ with entries in $\{0, 1\}$ such that
$\|x\|_2 =\sqrt{k}\|x\|_{\infty}$, and for any given $m$ and $k$ the
supports of the vectors were pairwise disjoint. We then hashed the
generated vectors using feature hashing, and recorded the $\ell_2$
norm of the embedded vectors.

The second phase then calculated the distortion between the original
and the embedded vectors, and computed the error probability
$\hat{\delta}$. Loosely speaking, $\hat{\delta}(m, k, \varepsilon)$ is the
ratio of the $2^{24}$ vectors for a given $m$ and $k$ that have
distortion greater than $\varepsilon$. Formally, $\hat{\delta}$ is calculated using the following formula
\begin{equation*}
  \hat{\delta}(m, k, \varepsilon)
  = \frac{\Bigl|\bigl\{ x : \|x\|_2 =\sqrt{k}\|x\|_{\infty},
    \big|\|A_mx\|_2^2 - \|x\|_2^2\big| \geq \varepsilon \|x\|_2^2 \bigr\}\Bigr|}
  {\bigl|\{ x : \|x\|_2 =\sqrt{k}\|x\|_{\infty} \}\bigr|} \;,
\end{equation*}
where $\varepsilon$ was varied over exponentially spaced values in the
range $[2^{-10}, 2^{-1}]$. Note that $\hat{\delta}$ tends to the
true error probability as the number of vectors tends to
infinity. Computing $\hat{\delta}$ yielded a series of 4-tuples
$(m, k, \varepsilon, \hat{\delta})$ which can be interpreted as given
target dimension $m$, $\ell_\infty/\ell_2$ ratio $1/\sqrt{k}$, distortion
$\varepsilon$, we have measured that the failure probability is at
most $\hat{\delta}$.

In the third phase, we varied $\delta$ over exponentially spaced
values in the range $[2^{-20}, 2^0]$, and calculated a value $\hat{\nu}$. Intuitively, $\hat{\nu}(m, \varepsilon, \delta)$ is the largest $\ell_\infty/\ell_2$ ratio
such that for all vectors having at most this $\ell_\infty/\ell_2$ ratio the measured error probability $\hat{\delta}$ is at most $\delta$. Formally, 
\begin{equation*}
  \hat{\nu}(m, \varepsilon, \delta) = \max \Bigl\{ \frac{1}{\sqrt{k}}
  : \forall k' \geq k, \hat{\delta}(m, k', \varepsilon) \leq \delta \Bigr\} \;.
\end{equation*}
Note once more that $\hat{\nu}$ tends to the
true $\nu$ value as the number of vectors tends to infinity.

To find a bound on the constant of the $\Theta$-notation in
\autoref{th:main}, we truncated data points that did not satisfy
$\frac{4 \lg \frac{1}{\delta}}{\varepsilon^2} \leq m <
\frac{2}{\varepsilon^2 \delta}$, and for the remaining points we
plotted $\hat{\nu}$ over the theoretical bound in \autoref{fig:both_cut}:
\begin{equation*}
  \frac{\hat{\nu}(m, \varepsilon, \delta)}
  {\min\biggl\{ \frac{\sqrt{\varepsilon}\lg\frac{\varepsilon
        m}{\lg \frac{1}{\delta}}}{\lg\frac{1}{\delta}},
    \sqrt{\frac{\varepsilon\lg \frac{\varepsilon^2 m}{\lg
          \frac{1}{\delta}}}{\lg\frac{1}{\delta}}} \biggr\}} \;.
\end{equation*}
From this plot we conclude that the constant is at least $0.6$ on the
large range on parameters we tested. However, the smallest values seem
to be outliers and come from a combination of very sparse vectors
($k = 7$) and high target dimension ($m = 2^{14}$). For the rest of
the data points the constant is at least $0.725$. While there are data
points where the constant is larger (i.e. feature hashing performs
better), there are data points close to $0.725$ over the entire range
of $\varepsilon$ and $\delta$.

\begin{figure}
  \centering
  \includegraphics[height=0.24\textheight]{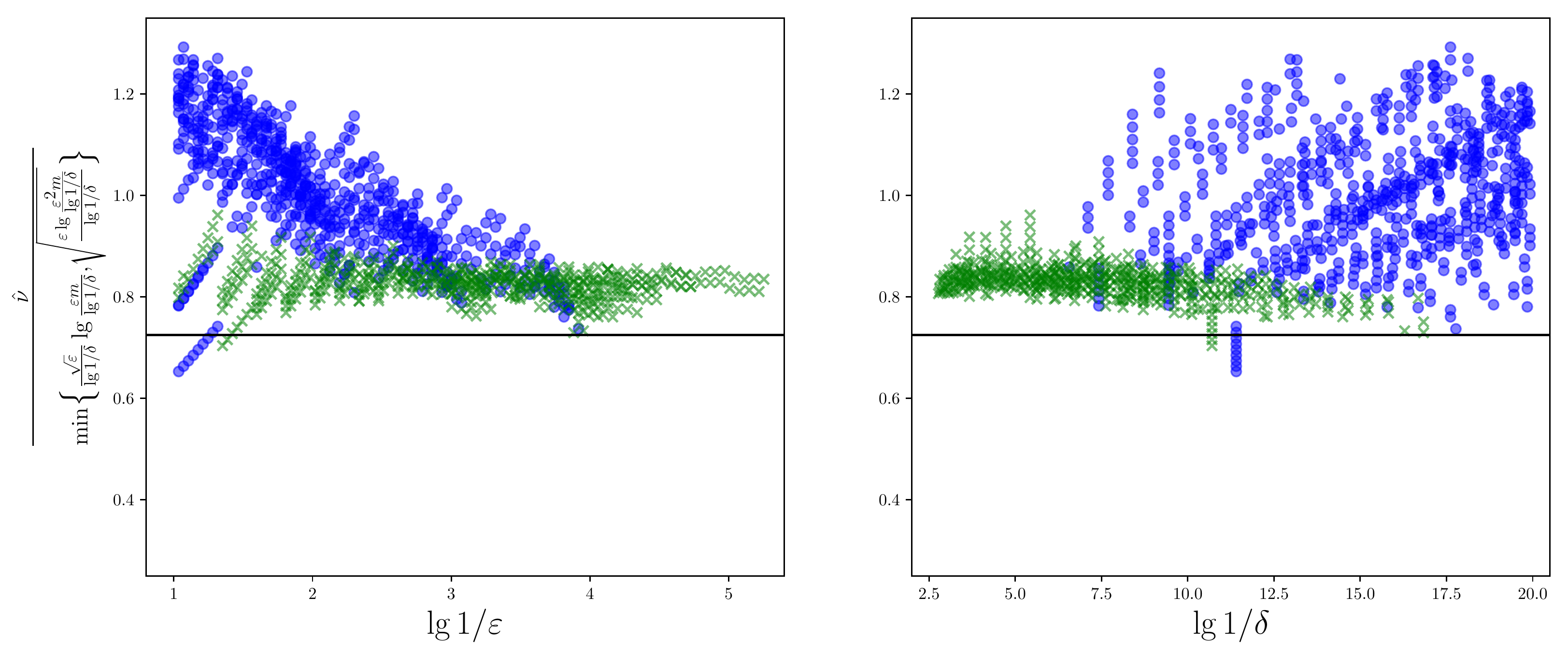}
  \caption{This plot shows the measured $\hat{\nu}$ values over the
    theoretical bound (abbreviated here):
    $\min\{\mathsf{left}, \mathsf{right}\}$. This ratio corresponds to
    the constant in the $\Theta$-notation in \autoref{th:main}. The
    points are marked with blue circles if
    $\mathsf{left} < \mathsf{right}$, otherwise they are marked with
    green $\times$'s. The horizontal line at 0.725 is there to ease
    comparisons with \autoref{fig:both_all}. The data points below the
    line come from very sparse vectors ($k = 7$) with high target
    dimension ($m = 2^{14}$).}
  \label{fig:both_cut}
\end{figure}

In \autoref{fig:both_all} we show that we indeed need both terms in
the minimum in \autoref{th:main}, by plotting the measured $\hat{\nu}$
values over both terms in the minimum in the theoretical bound
separately. For both terms there are points whose value is
significantly below $0.725$.

\begin{figure}[H]
  \centering
  \includegraphics[height=0.24\textheight]{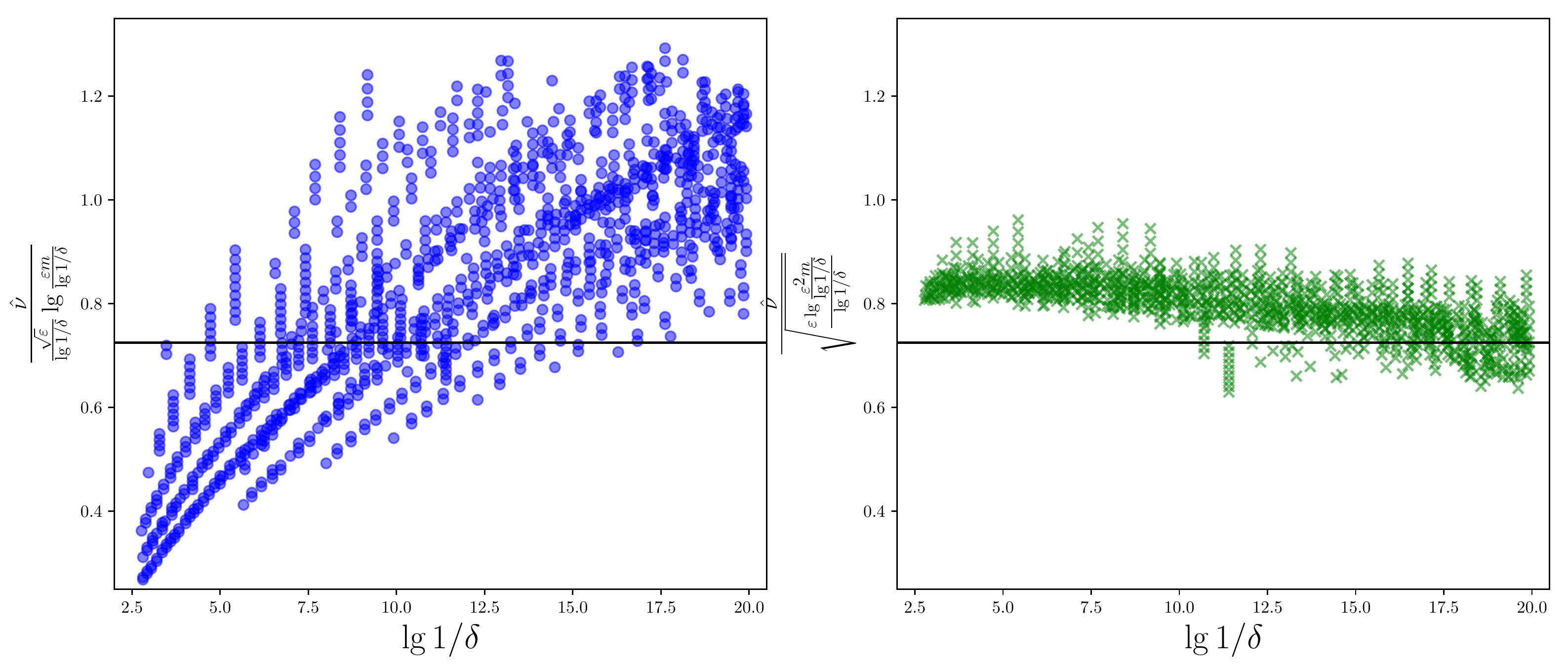}
  \caption{This plot shows the measured $\hat{\nu}$ values over each
    of the two terms in the minimum in the theoretical bound
    (abbreviated here): $\min\{\mathsf{left}, \mathsf{right}\}$. In
    the left subfigure the $y$-axis of the blue circles is
    $\frac{\hat{\nu}}{\mathsf{left}}$, while the $y$-axis of the green
    $\times$'s in the right subfigure is
    $\frac{\hat{\nu}}{\mathsf{right}}$. Note that the $x$-axis (values of $\lg(1/\delta)$ ) is the
    same in both subfigures, and the same as in the right subfigure of \autoref{fig:both_cut}. 
		As in \autoref{fig:both_cut}, the
    horizontal line at 0.725 is there to ease comparison between the
    figures.}
  \label{fig:both_all}
\end{figure}

To find a bound on $m$ where $\hat{\nu}(m,\varepsilon, \delta) = 1$ we
took the untruncated data and recorded the maximal $\hat{\delta}$ for
each $m$ and $\varepsilon$. We then plotted
$m \varepsilon^2 \hat{\delta}$ in \autoref{fig:border_1}. From
\autoref{fig:border_1} it is clear that
$\hat{\nu}(m, \varepsilon, \delta) = 1$ when
$m \geq \frac{2}{\varepsilon^2 \delta}$. Furthermore, the figure also
shows that there are data points where
$\hat{\nu}(m, \varepsilon, \delta) < 1$ while
$m = \frac{2 - \gamma}{\varepsilon^2 \delta}$, for some small
$\gamma$. Therefore we conclude the bound
$m \geq \frac{2}{\varepsilon^2 \delta}$ is tight.

\begin{figure}
  \centering
  \includegraphics[height=0.24\textheight]{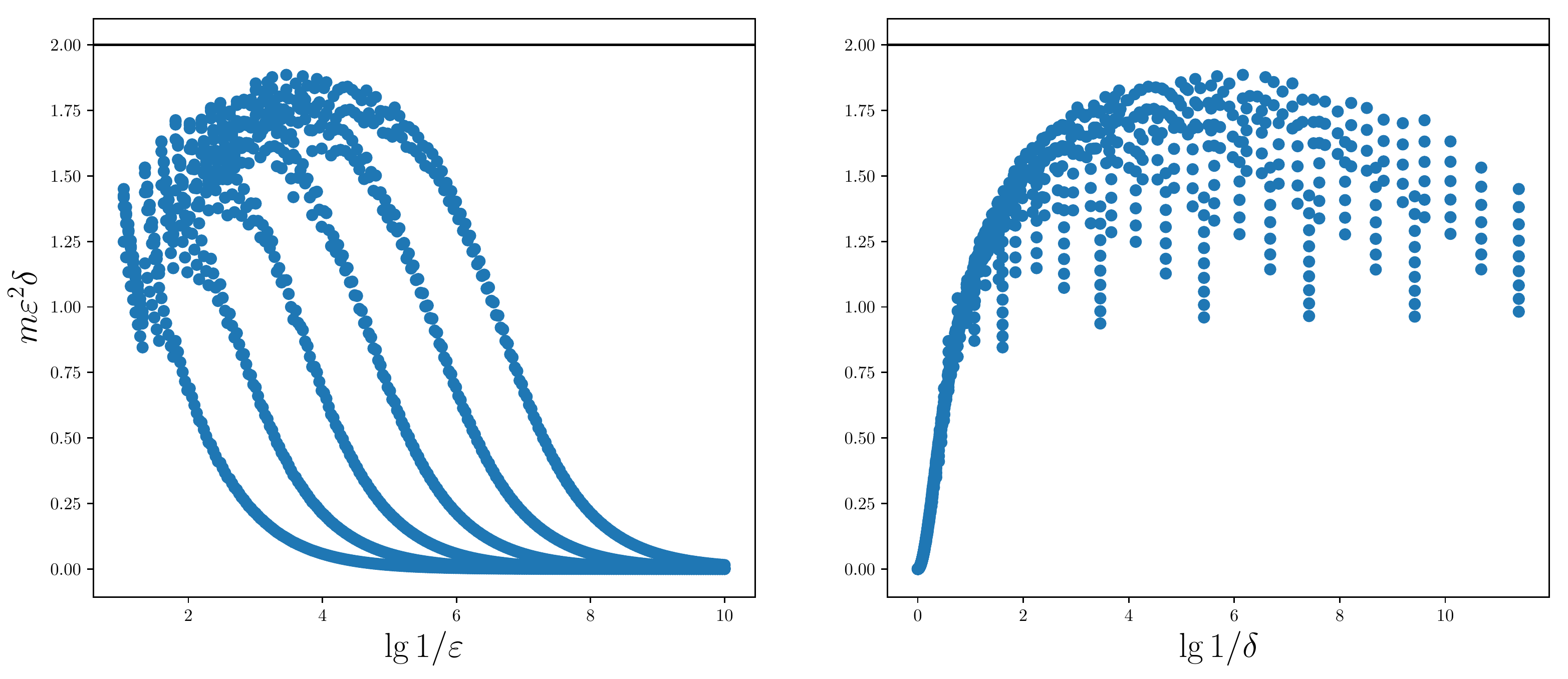}
  \caption{This plot shows the constant border where
    $\hat{\nu}(m, \varepsilon, \delta)$ becomes 1 for the first time. The
    theory states that if $2 \leq m \varepsilon^2 \delta$ then
    $\hat{\nu}(m, \varepsilon, \delta) = 1$. The distinct curves in the left
    plot correspond to distinct values of $m$.}
  \label{fig:border_1}
\end{figure}

\subsection{Implementation Details}
\label{sec:impl-deta}
As random number generators, we used degree 20 polynomials modulo the
Mersenne prime $2^{61} - 1$, where the coefficients were random data
from \url{random.org}. The random data was independent between
experiments with different values of $m$, and between the random number
generator used for vector generation and hashing.

Feature hashing was done using double tabulation hashing \cite{T13}
on 64 bit numbers. The tables in our implementation of double
tabulation hashing were filled with numbers from the aforementioned
random number generator. Double tabulation hashing has been proven to
behave fully randomly with high probability \cite{DKRT15}.


\newcommand{\etalchar}[1]{$^{#1}$}

\end{document}